\documentclass{article}

\PassOptionsToPackage{numbers,sort&compress}{natbib}
\usepackage[preprint]{neurips_2026}

\usepackage[utf8]{inputenc} 
\usepackage[T1]{fontenc}    
\usepackage[colorlinks=true,
            linkcolor=linkblue,
            citecolor=linkblue,
            urlcolor=linkblue]{hyperref}
\usepackage{url}            
\usepackage{booktabs}       
\usepackage{amsfonts}       
\usepackage{nicefrac}       
\usepackage{microtype}      
\usepackage{xcolor}         
\usepackage{amsmath}        %
\usepackage{amsthm}        %
\usepackage{enumitem}       %
\usepackage{placeins}
\usepackage{cleveref}       
    \crefname{figure}{Fig.}{Figs.}
    \crefname{table}{Tab.}{Tabs.}
    \crefname{section}{Sec.}{Secs.}
    \crefname{appendix}{App.}{Apps.}
    \crefname{equation}{Eq.}{Eqs.}
    \crefname{algorithm}{Alg.}{Algs.}
\usepackage{algorithm}
\usepackage{listings}
\usepackage[table]{xcolor}
\usepackage{graphicx}

\usepackage{fontawesome5}     
\usepackage[most]{tcolorbox}  

\definecolor{badgebg}{HTML}{F5F5F5}
\newtcbox{\linkbadge}{on line, boxrule=0pt, boxsep=0pt,
  colback=badgebg, arc=7pt,
  left=7pt, right=7pt, top=3.5pt, bottom=3.5pt,
  fontupper=\small\sffamily}
\newcommand{\reslink}[3]{\linkbadge{\href{#3}{\textcolor{black!65}{#1\;#2}}}}

\definecolor{linkblue}{HTML}{1A4E8A} 
\definecolor{codebg}{HTML}{FAFAFA}
\definecolor{codekw}{HTML}{6C3483}   
\definecolor{codefn}{HTML}{1F618D}   
\definecolor{codecm}{HTML}{7F8C8D}   
\definecolor{codenm}{HTML}{B9770E}   
\definecolor{codedoc}{HTML}{C0621B}  

\newtheorem{theorem}{Theorem}

\lstdefinestyle{pytorch}{
  language=Python,
  backgroundcolor=\color{codebg},
  basicstyle=\ttfamily\small,
  keywordstyle=\color{codekw}\bfseries,
  emphstyle=\color{codefn},
  emph={encoder,fwd_model,inv_model,F,mse_loss},
  commentstyle=\color{codecm}\itshape,
  stringstyle=\color{codedoc},
  numberstyle=\color{codenm},
  showstringspaces=false,
  frame=single,
  framerule=0pt,
  rulecolor=\color{codebg},
  xleftmargin=14pt, xrightmargin=14pt,
  framexleftmargin=14pt, framexrightmargin=14pt,
  framextopmargin=8pt, framexbottommargin=8pt,
  columns=fullflexible,
  keepspaces=true,
  aboveskip=4pt, belowskip=0pt,
  morestring=[s][\color{codedoc}]{"""}{"""},
}

\title{Sensorimotor World Models: Perception for Action via Inverse Dynamics}

\author{%
\begin{tabular}{c}
Petr Ivashkov\thanks{Correspondence to: \texttt{petr.ivashkov@tuebingen.mpg.de}.} \textsuperscript{ \normalfont1}
\quad
Randall Balestriero\textsuperscript{ \normalfont2}
\quad
Bernhard Schölkopf\textsuperscript{ \normalfont1,3,4}
\\[0.35em]
{\normalfont\textsuperscript{1}Max Planck Institute for Intelligent Systems, Tübingen, Germany}
\\
{\normalfont\textsuperscript{2}Department of Computer Science, Brown University, Providence, RI, USA}
\\
{\normalfont\textsuperscript{3}ELLIS Institute, Tübingen, Germany}
\\
{\normalfont\textsuperscript{4}ETH Zürich, Zürich, Switzerland}
\end{tabular}
}

\begin{document}

\maketitle

\vspace{-2em}
\begin{center}
  \reslink{\faGlobe}{Webpage}{https://petr-ivashkov.github.io/sensorimotor-world-model.github.io/}%
  \hspace{1em}%
  \reslink{\faGithub}{Code}{https://github.com/petr-ivashkov/sensorimotor-world-model}
\end{center}

\begin{abstract}
  {\em Perception for action} suggests that representations of the world should be shaped not by visual fidelity alone, but by their relevance for actions. At the same time, latent JEPA-style world models advocate learning compact predictive states from high-dimensional observations to facilitate the prediction of future states, but end-to-end training of these models is nontrivial because representations may collapse if our only goal is to construct a latent state that is easy to predict. We introduce a \emph{sensorimotor world model} (SMWM): a latent world model trained end-to-end with inverse dynamics regularization. This single regularizer addresses both issues: it prevents representation collapse and induces action-aligned representations. By forcing latent states to preserve information about the action underlying a transition, it biases the model toward the controllable degrees of freedom of the environment while discarding uncontrollable distractors. This yields stable latent world models trained from offline, reward-free trajectories, without frozen encoders, exponential moving averages, or complex latent regularizers. Empirically, SMWM learns compact, interpretable latent spaces and enables competitive planning performance across simple 2D and 3D control tasks.
\end{abstract}

\section{Introduction}\label{sec:introduction}

World models that predict future states given actions are central to intelligent agents \citep{ha2018world,hafner2019dream,DBLP:journals/corr/abs-1911-10500}. Learning such models directly in pixel space is difficult because observations are high-dimensional and pixel-level reconstruction encourages the model to capture background details that are often irrelevant for control. Latent world models avoid these issues by learning dynamics in a representation space rather than in observation space \citep{hafner2019dream,hafner2023mastering,hansen2022temporal,hansen2024tdmpc2}, and more recently, Joint Embedding Predictive Architectures (JEPAs) \citep{lecun2022path,assran2023self,bardes2024vjepa} dispense with reconstruction altogether and predict directly in embedding space.

\begin{figure}[hbt]
    \centering
\includegraphics[width=0.95\linewidth]{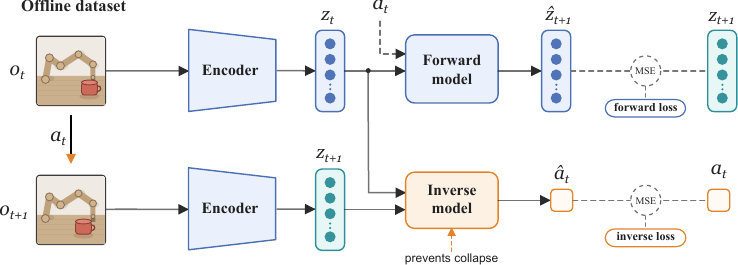}
    \caption{\textbf{Method overview.} We train an encoder $f_\theta$, a forward dynamics model $g_\phi$, and an inverse dynamics model $h_\psi$ jointly from an offline dataset of transitions $(o_t, a_t, o_{t+1})$. The encoder maps each observation to a compact embedding, $z_t = f_\theta(o_t)$ and $z_{t+1} = f_\theta(o_{t+1})$. The forward model predicts the next embedding from the current embedding and action, $\hat z_{t+1} = g_\phi(z_t, a_t)$, and is supervised by the mean-squared forward loss $\mathcal{L}_{\text{fwd}}$ between $\hat z_{t+1}$ and $z_{t+1}$. In parallel, the inverse model predicts the executed action from the two embeddings, $\hat a_t = h_\psi(z_t, z_{t+1})$, supervised by the mean-squared inverse loss $\mathcal{L}_{\text{inv}}$ between $\hat a_t$ and $a_t$. Both losses backpropagate into the encoder through the joint objective $\mathcal{L} = \mathcal{L}_{\text{fwd}} + \lambda\,\mathcal{L}_{\text{inv}}$. The inverse loss acts as an anti-collapse mechanism: to enable action recovery, $f_\theta$ must preserve the action-relevant information contained in $(o_t, o_{t+1})$.} 
    \label{fig:schematic}
\end{figure}

JEPA-style world models from pixels are appealing but susceptible to \emph{representation collapse}: when the encoder and the latent dynamics model are trained jointly to make correct predictions in embedding space, the encoder may learn to map all observations to a constant embedding, making the dynamics prediction task trivial while rendering the learned model useless for other tasks \citep{grill2020bootstrap,bardes2022vicreg}. Recent work approaches this problem from several angles. DINO-WM \citep{zhou2025dinowm} sidesteps the problem by freezing a pretrained encoder. PLDM \citep{sobal2025pldm} adds a multi-term variance--covariance regularizer on the embeddings. LeWorldModel \citep{maes2026leworldmodel,balestriero2025lejepa} matches the embedding distribution to an isotropic Gaussian via SIGReg. V-JEPA~2 \citep{assran2025vjepa2} uses stop-gradient targets from an exponential-moving-average target encoder.

While representation collapse is an important practical problem, it also points to the basic question of what a useful world model representation should retain. It has been argued that this connects to causality: a causal representation \cite{Scholkopfetal21} should not merely encode predictive correlations, but also actions and their effects, to provide an interventional world model \cite{DBLP:journals/corr/abs-1911-10500}. In cognitive psychology and neuroscience, the term {\em perception for action} captures the view that perception ultimately subserves action~\cite{Goodale-Milner}. 
{\em Common coding} proposes that perceived events and planned actions share a unified representational domain, with actions represented partly in terms of their anticipated perceptual effects \cite{Prinz1990}, rendering the content of a representation dependent on the organism's action repertoire.\footnote{This also relates to Gibson's notion of {\em affordances}: organisms perceive the environment in terms of possibilities for action, such as grasping, walking, avoiding, or manipulating \cite{Gibson1979}. Going further back, von Uexküll's concept of the \emph{Umwelt} describes the organism-specific world disclosed by a creature's sensory and motor capacities \cite{Uexkull1934}. Ultimately, the world should thus look rather different to different creatures, even if they appear to inhabit the same environment.} Representations should thus not only be judged by how much information they preserve about observations, but by whether they preserve the distinctions that matter for the actions available to the agent and for the consequences of those actions.

{\em Enactive perception} is the theory that perception is not a passive reception of sensory input, but an active bodily exploration of the world \cite{VarelaThompsonRosch1991}. A recent instantiation of this view focuses on {\em sensorimotor contingencies} as the lawful relationships between an agent's actions and the resulting changes in sensory input \cite{OReganNoe2001}. With this in mind, we return to our practical problem, i.e., collapse: can we train a latent world model whose representation respects these contingencies by retaining information sufficient to recover the agent's actions? We call the resulting object a \emph{sensorimotor world model} (SMWM): a predictive latent model whose state is shaped by the coupling between motor commands and sensory change.

We use a simple and principled mechanism to instantiate this idea: \emph{inverse dynamics regularization}. We add a single-step inverse dynamics head that predicts the action taken between two consecutive observations from their embeddings, and we propagate gradients from this head into the encoder, as illustrated in \cref{fig:schematic}. The intuition is direct: if two consecutive embeddings contain enough information to recover the action that produced the transition, they must preserve features of the observations that are relevant for controllable dynamics. This encourages the encoder to retain action-relevant information while discarding the rest. Unlike distributional regularizers that prescribe the geometry of the embedding space, inverse dynamics anchors the representation to a task-grounded quantity---the action.

\paragraph{Setting and contributions.} We consider offline datasets of trajectories consisting of video frames and corresponding continuous actions, without rewards, task labels, or knowledge of the data-collecting policy. This setting broadly covers demonstration data, action-annotated video datasets, and logged interaction data. Our contributions are: 
\begin{enumerate}
    \item SMWM, a simple \emph{sensorimotor world model} trained end-to-end from pixels, using inverse dynamics regularization as the sole anti-collapse mechanism with one additional hyperparameter,
    \item evidence that the learned representations are not merely non-collapsed, but track the controllable degrees of freedom, preserve spatial geometry, and filter out uncontrollable distractors,
    \item competitive planning performance against SIGReg-regularized baselines across 2D and 3D control tasks, showing that this structural advantage carries over to downstream control.
\end{enumerate}

\section{Background and related work}\label{sec:related}

\paragraph{Latent world models.} Learning dynamics models in learned latent spaces has a rich history. The Dreamer family \citep{hafner2019dream,hafner2020mastering,hafner2023mastering} and PlaNet \citep{hafner2019learning} learn latent dynamics with a pixel decoder under a variational objective. TD-MPC \citep{hansen2022temporal,hansen2024tdmpc2} drops the decoder and trains a joint encoder and latent dynamics model end-to-end with a temporal-difference objective. JEPAs \citep{lecun2022path,assran2023self,bardes2024vjepa} push this further, dispensing with both reconstruction and reward signal and predicting directly in embedding space. Our work belongs to the JEPA family and operates in the offline, reward-free regime.

\paragraph{Representation collapse.} The self-supervised learning literature has charted a design space of remedies for representation collapse: contrastive losses with negative samples \citep{oord2018representation,chen2020simple}, asymmetric architectures with stop-gradients \citep{grill2020bootstrap,chen2021exploring}, and regularizers on the embedding distribution \citep{zbontar2021barlow,bardes2022vicreg}. Recent JEPA world models inherit choices from this design space: DINO-WM \citep{zhou2025dinowm} freezes a pretrained encoder, PLDM \citep{sobal2025pldm} adds multiple auxiliary loss terms including variance--covariance regularization, temporal smoothness, and inverse dynamics, LeWorldModel \citep{maes2026leworldmodel,balestriero2025lejepa} matches embeddings to an isotropic Gaussian, and V-JEPA~2 \citep{assran2025vjepa2} uses stop-gradient and exponential moving average target networks. We share the JEPA training framework and the offline pixel-action setting with these works but use inverse-dynamics as the standalone anti-collapse mechanism, as described in \cref{sec:method}.

\paragraph{Inverse dynamics.} The use of inverse dynamics for representation learning is not new.  \citet{pathak2017curiosity} use it to learn features that capture controllable aspects of the environment for intrinsic-reward computation. A line of theoretical work shows that, in finite rich-observation settings with exogenous distractors, multi-step inverse objectives can recover control-relevant latent state information under certain assumptions  \citep{efroni2022provable,lamb2023guaranteed,islam2022agent,mhammedi2023musik}. While none of these results transfers to our continuous-state, continuous-action setting with a single-step inverse, we treat them as motivation for using action prediction as an inductive bias toward controllable state structure. Concurrent work has begun to combine inverse dynamics with world-model objectives in related but distinct settings \citep{han2026wam,terver2026ebjepa,yu2025auxjepa}; we discuss our differences from each of these works in \cref{app:extended}.

\paragraph{Other directions.} Several adjacent lines share architectural ingredients but address different problems. Action-free or latent-action world models infer pseudo-actions from video and use them to supervise dynamics: LAPO \citep{schmidt2024learning} and DynaMo \citep{cui2024dynamo} use a learned inverse dynamics model to extract latent actions from action-free demonstrations, and APV \citep{seo2022apv} pretrains an action-free latent video predictor for downstream action-conditional world-model learning. Visual RL methods share encoders between auxiliary objectives and policies but do not learn explicit world models \citep{laskin2020curl,kostrikov2020image,yarats2021reinforcement}. Bisimulation metrics offer a complementary lens on which states should map to similar embeddings \citep{ferns2004metrics,castro2020scalable}, and causal abstraction \citep{Rubensteinetal17} formalizes a related consistency between structural models via a commutative diagram on states and interventions. \cref{app:extended} provides further detail.

\section{Method}\label{sec:method}

\paragraph{Dataset and setting.}
We assume access to an offline dataset $\mathcal{D} = \{(o_t, a_t, o_{t+1})\}$ of transitions, where $o_t, o_{t+1} \in \mathcal{O}$ are consecutive observations (e.g., video frames) and $a_t \in \mathcal{A} \subseteq \mathbb{R}^m$ is a continuous action taken in between. We do not require a behavior policy or a reward function. Our goal is to learn an encoder $f_\theta: \mathcal{O} \rightarrow \mathcal{Z}$ that maps observations to a compact embedding space $\mathcal{Z} \subseteq \mathbb{R}^d$, together with a forward dynamics model $g_\phi: \mathcal{Z} \times \mathcal{A} \rightarrow \mathcal{Z}$ that predicts the next embedding from the current embedding and action.

\paragraph{Forward dynamics in embedding space.}
Given a transition $(o_t, a_t, o_{t+1})$, we encode both observations and train the dynamics model to match the embedding of the next observation,
\begin{equation}
    z_t = f_\theta(o_t), \quad z_{t+1} = f_\theta(o_{t+1}), \quad \hat{z}_{t+1} = g_\phi(z_t, a_t),
\end{equation}
by minimizing the mean-squared forward loss
\begin{equation}
    \mathcal{L}_{\text{fwd}} = \mathbb{E}_{(o_t, a_t, o_{t+1}) \sim \mathcal{D}} \left[ \| \hat{z}_{t+1} - z_{t+1} \|_2^2 \right].
\end{equation}
This objective alone admits a trivial solution: if $f_\theta$ maps every observation to a constant embedding $z^\star$ and $g_\phi$ outputs $z^\star$ regardless of its inputs, then $\mathcal{L}_{\text{fwd}} = 0$. Such a collapsed representation carries no information about the observation and is useless for downstream planning or control.

\paragraph{Inverse dynamics regularization.}
To prevent this collapse, we introduce an inverse dynamics model $h_\psi: \mathcal{Z} \times \mathcal{Z} \rightarrow \mathcal{A}$ that predicts the action taken between two consecutive observations from their embeddings, $\hat{a}_t = h_\psi(z_t, z_{t+1})$, trained with the mean-squared action loss
\begin{equation}
    \mathcal{L}_{\text{inv}} = \mathbb{E}_{(o_t, a_t, o_{t+1}) \sim \mathcal{D}} \left[ \| \hat{a}_t - a_t \|_2^2 \right].
\end{equation}
Under squared loss, a collapsed encoder would give the inverse model identical inputs for every transition, so the best inverse predictor is a constant action, with risk $\mathbb{E}\|a_t-\mathbb{E}[a_t]\|_2^2$. Any reduction below this constant-predictor risk requires $(z_t,z_{t+1})$ to preserve information that is predictive of $a_t$. Such a nontrivial inverse prediction thus rules out a fully collapsed representation. In contrast to distributional priors such as SIGReg's isotropic-Gaussian regularizer~\citep{maes2026leworldmodel}, this mechanism does not prescribe the geometry of the embedding space: it only requires that the encoder preserve whatever information about $a_t$ is contained in the pair $(o_t, o_{t+1})$.

\paragraph{Joint training objective.}
We jointly optimize the encoder, forward model, and inverse model by minimizing
\begin{equation}
    \mathcal{L} = \mathcal{L}_{\text{fwd}} + \lambda \, \mathcal{L}_{\text{inv}},
    \label{eq:total-loss}
\end{equation}
where $\lambda > 0$ controls the strength of the inverse dynamics regularization. The forward loss updates $f_\theta$ and $g_\phi$; the inverse loss updates $f_\theta$ and $h_\psi$. The encoder, therefore, receives gradients from both objectives, producing embeddings that are simultaneously predictive under forward dynamics and action-informative under inverse dynamics.

\section{Learned latent structure}\label{sec:dotworld}

To build intuition for what SMWM learns and to assess which structure of the world the resulting representations reflect, we instantiate the method on a minimal testbed with known ground-truth state, intrinsic dimension, and action geometry.

\paragraph{A controlled testbed.}
We consider a \emph{dot world}: a single-dot environment whose state is fully described by the 2D position $(x, y)$ of a colored dot rendered on a $64{\times}64$ white canvas; the observation $o_t \in \mathbb{R}^{3 \times 64 \times 64}$ is the rendered image. Actions $a_t = (\Delta x, \Delta y) \in \mathcal{A} \subset \mathbb{R}^2$ displace the dot, so the dynamics are $s_{t+1} = s_t + a_t$ in the state space. The dataset $\mathcal{D}$ is constructed by sampling, for each transition, a position and a random displacement. The encoder $f_\theta$ is a small CNN with the latent dimension $d = 64$, and the forward $g_\phi$ and inverse $h_\psi$ models are two-layer MLPs. The exact architecture, dataset, and training details are deferred to \cref{app:dotworld}.

\paragraph{Recovered intrinsic dimension.}
First, we show that $f_\theta$ identifies the true 2D structure of the world from pixels and actions alone. \cref{fig:dotworld-pca} reports the PCA spectrum of the embeddings $\{f_\theta(o)\}$ over a uniform grid of dot positions. Two principal components carry essentially all the variance in the embedding space, and the spectrum drops sharply at the true intrinsic dimension $d_{\text{true}} = 2$; the remaining $62$ directions are effectively collapsed. Projecting the embeddings onto the top two principal components reveals that the latent geometry is not only low-dimensional but also \emph{spatially faithful}: the grid of world states maps to a near-square grid in latent space. A mild deformation occurs only near the boundary, which corresponds to out-of-distribution states. As a result, spatial neighbors in the observation space remain neighbors in the latent space. The encoder thus recovers an approximately 2-dimensional, topology-preserving representation of the state space without any ground-truth state supervision.

\begin{figure}[tb]
    \centering
    \includegraphics[width=0.95\linewidth]{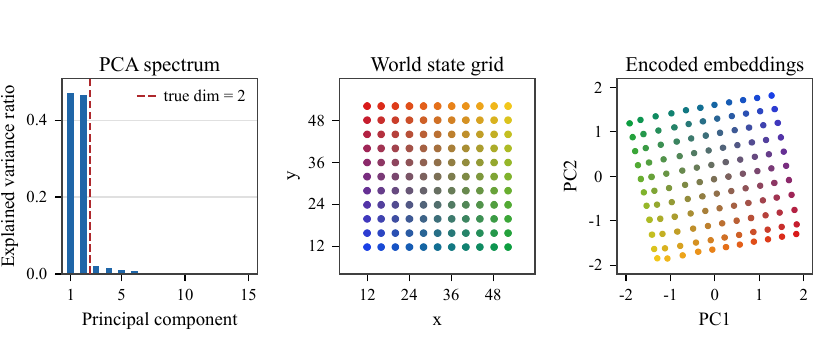}
    \caption{\textbf{Dot world latent geometry.} \textbf{Left:} PCA spectrum of the learned embeddings; the explained-variance ratio drops sharply past the true intrinsic dimension $d_{\text{true}} = 2$ (red dashed line). \textbf{Center:} grid of probe world states $(x, y)$, color-coded by position. \textbf{Right:} the same probes embedded by $f_\theta$ and projected onto the top two principal components. Despite no state supervision, the encoder recovers an effectively 2-dimensional representation inside a 64-dimensional ambient latent space. Moreover, the recovered representation is spatially faithful to the true state space. The latent grid appears rotated relative to $(x, y)$, as expected: PC axes are identifiable only up to rotation and sign.
    }
    \label{fig:dotworld-pca}
\end{figure}

\paragraph{Forward model.}
The encoder and forward model should jointly satisfy a consistency condition: encoding the observation reached after taking action $a$ should agree with first encoding the current observation and then applying $g_\phi(\cdot, a)$. For ease of notation, we drop the indices and write $g_a(z) := g_\phi(z, a)$ and $f$ for $f_\theta$, rendering the consistency condition
\begin{equation}
    f(a(o)) \overset{!}{=} g_{a}(f(o)).
    \label{eq:homomorphism}
\end{equation}
In other words, the diagram on the left of \cref{fig:dotworld-rollout} must commute \citep{Hertz1899,Rubensteinetal17}. To test it, we encode an initial observation $o_1$, unroll $g$ autoregressively along a 5-step action sequence $a_1,\dots a_5$, and compare the predicted embeddings against the encoded ground-truth embeddings $z_t = f(o_t)$. The predicted and encoded trajectories agree closely along the entire rollout (\cref{fig:dotworld-rollout}, right), illustrating that \cref{eq:homomorphism} approximately holds on the trained model. Each $a_t \in \mathcal{A}$ induces a corresponding displacement $\rho(a_t)$ in the relevant subspace of the latent space (with $\rho(a)$ approximately independent of $z$), so the forward model behaves as $g_a(z) \approx z + \rho(a)$. The trained model thus delivers, in addition to the state representation $f$, a representation of actions as latent translations $g_a$, without any term in $\mathcal{L}$ enforcing such a structure.

\begin{figure}[tb]
    \centering
    \includegraphics[width=0.95\linewidth]{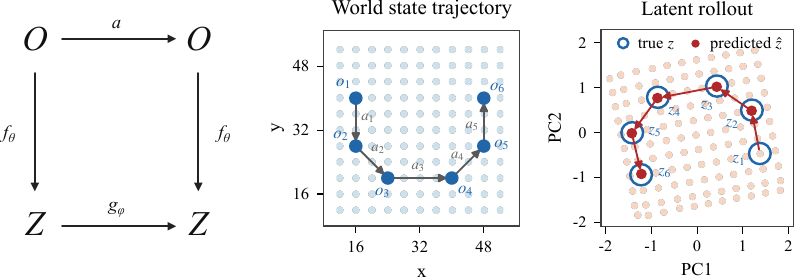}
    \caption{\textbf{Encoder and forward model commute.} \textbf{Left:} Equivariance that should be satisfied by the learned representation: $f \circ a = g_a \circ f$. \textbf{Center:} a 5-step trajectory in world-state space with actions $a_1, \dots, a_5$. \textbf{Right:} the corresponding rollout in latent space; predictions $\hat z_t$ (filled red) obtained by autoregressive application of $g$ track the encoded ground-truth embeddings $z_t = f(o_t)$ (open blue) along the entire rollout. Therefore, the encoder and forward model approximately commute, and world actions act as translations in latent space.}
    \label{fig:dotworld-rollout}
\end{figure}

\paragraph{Equivariance and latent action composition.}
The above approximate equivariance
condition is encouraged by the forward prediction loss. Denote the effect of action $a$ on the observation $o$ by $o\mapsto a(o)$.\footnote{We overload notation: in some cases, $a\in\mathcal A\subseteq\mathbb R^m$ denotes an action vector.} The forward loss minimizes, in our simplified notation,
\[
    \left\| f(a(o)) - g_a(f(o)) \right\|_2^2 .
\]
On the support of the training distribution and up to approximation error, this encourages
\[
    f(a(o)) \approx g_a(f(o)).
\]
This means that encoding after applying the intervention approximately agrees with first encoding and then applying the learned latent intervention $g_a$.\footnote{The encoder thus intertwines the physical intervention with the learned
latent intervention, making this a statement of {\em equivariance}. In physics terminology one might also say that the latent state transforms
{\em covariantly} with the intervention.}

This one-step equivariance has a compositional consequence. If actions form a semigroup under composition and equivariance holds exactly, then physical action composition must be represented by composition of the learned latent interventions, at least on the encoded data manifold:

\begin{theorem}
\label{thm:equivariance-homomorphism}
Let \(\mathcal A\) be a semigroup. If \(f\) is equivariant, i.e., $f(a(o)) = g_a(f(o))$ for all \(a\in\mathcal A\) and \(o\in\mathcal O\), then \(a\mapsto g_a\) is a homomorphism on \(f(\mathcal O)\). That is, for all \(a_1,a_2\in\mathcal A\) and \(z\in f(\mathcal O)\),
\begin{equation}
    g_{a_2\circ a_1}(z)
    =
    g_{a_2}\!\left(g_{a_1}(z)\right).
\end{equation}
\end{theorem}

We prove \cref{thm:equivariance-homomorphism} in \cref{app:equivariance-composition}. In the learned model, this identity should be interpreted approximately.

Equivariance alone, however, admits degenerate solutions, including collapse. Useful representations, therefore, additionally require some form of \emph{faithfulness} of $g$ on $f({\cal O})$, i.e., every $a \in \cal A$ that acts nontrivially on $\cal O$ must also act nontrivially on $f({\cal O})$. More strongly, we can require that the latent action is \emph{identifiable from observations}: for every pair $(z, z') \in f({\cal O}) \times f({\cal O})$ with $z' = g_a(z)$, the action $a$ is recoverable from $(z, z')$. We achieve this by training an \emph{inverse} model that predicts $a$ from $(z, z')$.

Beyond a homomorphism into latent transformations under composition, the experiments in \cref{fig:dotworld-rollout} suggest that in our case the homomorphism may (approximately) go into the additive group, i.e., the learned interventions appear to act approximately as translations,
\[
    g_a(z) \approx z+\rho(a),
\]
with $\rho(a)$ approximately independent of $z$.  It is intriguing to ask where it may stem from. The inverse-dynamics loss provides a plausible bias toward this translation-like regime: a simple way to make \(a_t\) recoverable from \((z_t,z_{t+1})\) is to encode the action in the displacement vector
\begin{equation}
    z_{t+1}-z_t \approx \rho(a_t).
\end{equation}
If $\rho$ is injective on the action support, then the inverse $h$ can decode the
action by applying an approximate inverse of $\rho$ to this displacement.  This
mechanism does not enforce additivity, but it renders an additive, translation-like
representation a low-complexity solution.

\paragraph{Controllable degrees of freedom.}
The 2D dot world establishes that the learned representation has the right intrinsic dimension when it equals the action dimension. To stress-test this property, we vary the controllable degrees of freedom and add uncontrollable distractors. \cref{fig:dotworld-multidot} reports four configurations. \emph{Independent} contains two dots controlled by independent displacements, giving a 4D state and a 4D action. \emph{Coupled} contains two dots that move together under a single 2D displacement. \emph{Distractor} contains one controlled dot together with a second dot that moves randomly on its own and is not controlled by the action. \emph{Combined} merges all three: two independent dots, one coupled pair, and one distractor, with a 6D action vector.

In every case, the number of significant principal components matches the controllable dimension of the system, with the remaining directions effectively collapsed. One may notice that the spectra within the controllable subspace are not exactly uniform, even though the state space is uniformly covered under uniform coordinate sampling. However, this is expected: the encoder is free to rescale individual dimensions and concentrate variance unevenly, so what matters is the rank of the embedding, not the uniformity of its spectrum. Of the four configurations, the distractor case is particularly informative: even though the random dot's position varies in the observation, the encoder ignores it. This is the expected outcome of \cref{eq:total-loss}: a distractor's state is neither needed to recover the action from a transition (it is independent of $a_t$) nor predictable from past latent state under $g_\phi$ (since it evolves stochastically), so encoding it would only inflate $\mathcal{L}_{\text{fwd}}$ without reducing $\mathcal{L}_{\text{inv}}$. 

\begin{figure}[tb]
    \centering
    \includegraphics[width=0.95\linewidth]{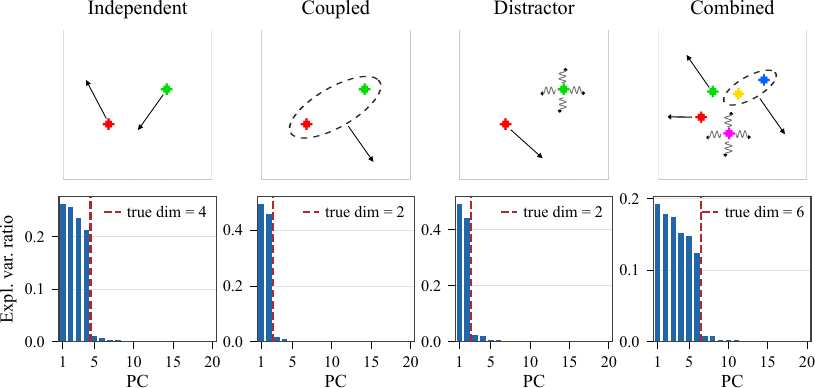}
    \caption{\textbf{Effective latent dimension tracks controllable degrees of freedom.} \textbf{Top row:} four dot-world configurations with controllable dimensions $4$, $2$, $2$, and $6$; in \emph{Distractor} and \emph{Combined}, the wavy-arrowed dot moves randomly and is not controlled by the action. \textbf{Bottom row:} PCA spectra of the corresponding learned embeddings, with the true intrinsic dimension marked by the red dashed line. The encoder allocates significant variance to exactly the controllable degrees of freedom and filters out uncontrollable distractors, correctly recovering a minimal necessary subspace inside the higher-dimensional ambient latent space.}
    \label{fig:dotworld-multidot}
\end{figure}

\paragraph{Visualizing control-dependent perception.}
As a complementary visualization, we ask whether the model ``sees'' the same object differently depending on what is controllable. To make this explicit, we introduce an additional toy environment in \cref{app:spriteworld}: an asymmetric 2D triangular sprite whose pose is \((x,y,\theta)\). The sprite can translate and rotate, but we vary the action repertoire by exposing only some of the pose changes as actions, while the remaining degrees of freedom continue to move randomly. After training, we freeze the encoder and train a post-hoc decoder from \(z_t\) to \(o_t\). The reconstructions show that the representation preserves the control-relevant pose variables and averages out the uncontrolled ones: for example, when agent rotations are uncontrolled, the decoded triangle becomes an orientation-averaged blob, whereas full \(x/y/\theta\) control yields a sharp oriented reconstruction.

Together, these experiments show that SMWM learns representations that are low-dimensional, spatially faithful, and structure-preserving with respect to the action-induced dynamics.

\section{Planning in latent space}\label{sec:planning}

The latent structure recovered above is only useful insofar as it carries over to richer environments and supports actionable downstream control.

\paragraph{Setup.} We perform goal-conditioned trajectory optimization directly in the learned latent space, following the planning setup of LeWM~\citep{maes2026leworldmodel}. Given a current observation \(o_1\) and a goal observation \(o_g\), we encode both with the frozen encoder, \(z_1 = f_\theta(o_1)\) and \(z_g = f_\theta(o_g)\), and roll the forward model out autoregressively from \(\hat z_1 = z_1\), \(\hat z_{t+1} = g_\phi(\hat z_t, a_t)\), to predict the latent trajectory induced by a candidate action sequence \(a_{1:H}\) over a planning horizon \(H\). Each candidate is scored by the terminal goal-matching cost
\begin{equation}
    \mathcal{C}(a_{1:H}) = \| \hat z_{H+1} - z_g \|_2^2,
    \label{eq:planning-cost}
\end{equation}
which we minimize with the Cross-Entropy Method (CEM)~\citep{rubinstein1997optimization}: at each iteration, CEM samples action sequences from a Gaussian, scores them via latent rollouts under \(g_\phi\), and refits the sampling distribution to the top-scoring fraction. Because autoregressive rollouts accumulate prediction error as \(H\) grows, we wrap CEM in a receding-horizon model-predictive control (MPC) loop, executing only the first \(K\le H\) optimized actions before re-encoding the resulting observation and replanning. Throughout, \(f_\theta\) and \(g_\phi\) remain frozen; only the action sequence is optimized. 

\paragraph{Environments.} We evaluate on four tasks shown in \cref{fig:main-planning}: TwoRoom (2D navigation), Reacher (continuous control), Push-T (2D contact-rich manipulation), and OGBench-Cube (3D tabletop manipulation). Datasets and environment details are deferred to \cref{app:environments}.

\paragraph{Baselines.} We compare against three baselines. \texttt{SIGReg}~\citep{maes2026leworldmodel} is our main head-to-head: it shares the same encoder and predictor but replaces \(\mathcal{L}_{\text{inv}}\) with a Gaussian-matching regularizer on the latent embeddings. \texttt{Forward-only} (\(\lambda = 0\)) ablates the regularizer entirely, isolating the contribution of our anti-collapse mechanism. \texttt{Random} samples actions uniformly from the action space and serves as a floor reference. Full implementation details are in the appendix.

\paragraph{Architecture.} We adopt LeWM's encoder and predictor~\citep{maes2026leworldmodel}: the encoder is a ViT-Tiny (\(\sim\)5M parameters) whose final \textsc{cls} token is projected to \(z \in \mathbb{R}^d\) with $d=192$, and \(g_\phi\) is a small transformer (\(\sim\)10M parameters) that predicts \(\hat z_{t+1}\) from a single \((z_t, a_t)\) pair, with \(a_t\) injected via AdaLN-zero. Our only architectural addition is the inverse model \(h_\psi\): a 2-layer MLP of width \(256\) mapping \([z_t;\, z_{t+1}] \in \mathbb{R}^{2d}\) to \(\hat a_t \in \mathbb{R}^m\). Details are provided in \cref{sec:implementation}.

\begin{figure}[tb]
    \centering
    \includegraphics[width=0.95\linewidth]{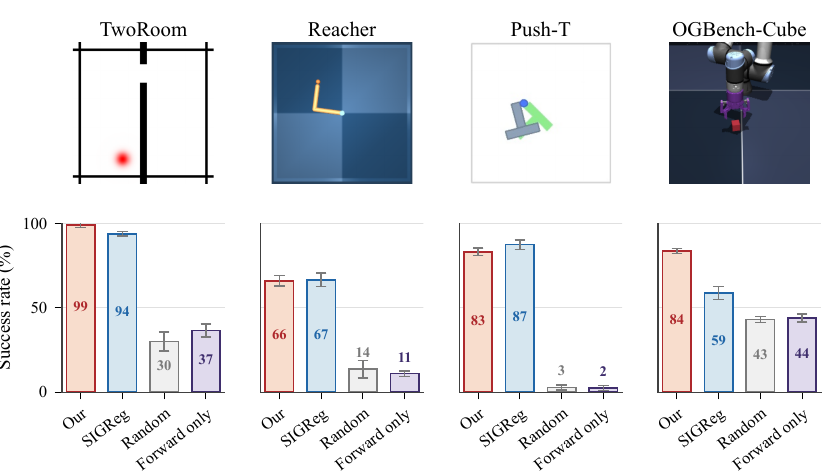}
    \caption{\textbf{Planning success across environments.} \textbf{Top:} the four evaluation environments---TwoRoom (2D navigation), Reacher (continuous control), Push-T (2D contact-rich manipulation), and OGBench-Cube (3D tabletop manipulation). \textbf{Bottom:} goal-conditioned planning success rate (mean and standard error over five seeds) under a fixed budget of \(50\) environment steps and a goal placed \(25\) steps ahead of the initial state. SMWM matches the SIGReg baseline on the three 2D tasks and clearly outperforms it on OGBench-Cube. Both regularized variants dominate the \emph{Forward-only} ablation and the \emph{Random}-action baseline, confirming that an effective anti-collapse mechanism is necessary for usable planning.}
    \label{fig:main-planning}
\end{figure}

\paragraph{Planning performance.} All methods are evaluated with a fixed budget of \(50\) environment steps per episode and goals placed \(25\) steps ahead of the initial state. Both regularized variants substantially outperform Forward-only and Random across all four environments (\cref{fig:main-planning}). SMWM approximately matches SIGReg on the three 2D tasks---TwoRoom, Reacher, and Push-T---and clearly outperforms it on OGBench-Cube, the only environment with a 3D contact-rich manipulation, where SMWM retains \(84\%\) success while SIGReg drops to \(59\%\).

\paragraph{Probing physical state.} To assess what physical content the embeddings retain, we train linear and two-layer MLP probes on frozen embeddings to regress the ground-truth state of each environment; held-out \(R^2\) values are reported in \cref{tab:physical_quantity_probe_r2}. 

\begin{table}[tb]
\centering
\caption{\textbf{Probing the physical state.} Held-out \(R^2\) for linear and two-layer MLP probes regressing ground-truth physical quantities from frozen embeddings; the best score per row and probe type is in bold. Both regularized methods (SMWM and LeWM/SIGReg) recover the underlying physical state near-perfectly under MLP probes, with the exception of Push-T block pose.}
\label{tab:physical_quantity_probe_r2}
\small
\setlength{\tabcolsep}{5pt}
\begin{tabular}{llrrrrrr}
\toprule
Environment & Quantity & \multicolumn{3}{c}{Linear probe} & \multicolumn{3}{c}{MLP probe} \\
\cmidrule(lr){3-5}\cmidrule(lr){6-8}
 & & SMWM & SIGReg & Fwd.-only & SMWM & SIGReg & Fwd.-only \\
\midrule
TwoRoom & agent position & \textbf{1.000} & 0.996 & 0.832 & \textbf{1.000} & 1.000 & 0.991 \\
\arrayrulecolor{black!35}\midrule\arrayrulecolor{black}
Reacher & joint angles & 0.946 & \textbf{0.999} & 0.647 & 0.995 & \textbf{1.000} & 0.962 \\
\arrayrulecolor{black!35}\midrule\arrayrulecolor{black}
Push-T & agent position & \textbf{0.993} & 0.954 & 0.293 & \textbf{0.988} & 0.986 & 0.433 \\
 & block position & 0.946 & \textbf{0.973} & 0.352 & 0.981 & \textbf{0.996} & 0.623 \\
 & block orientation & 0.664 & \textbf{0.931} & 0.202 & 0.882 & \textbf{0.987} & 0.491 \\
\arrayrulecolor{black!35}\midrule\arrayrulecolor{black}
Cube & cube position & \textbf{0.998} & 0.975 & 0.706 & \textbf{0.998} & 0.993 & 0.854 \\
 & gripper position & \textbf{0.999} & 0.986 & 0.902 & \textbf{0.999} & 0.995 & 0.979 \\
 & gripper opening & \textbf{0.989} & 0.961 & 0.269 & \textbf{0.992} & 0.983 & 0.427 \\
\bottomrule
\end{tabular}
\end{table}

Both regularized methods recover the underlying state near-perfectly under MLP probes, and their linear-probe scores stay near saturation across most quantities, with SMWM holding a consistent linear advantage on the Push-T agent position and on all three Cube quantities, and SIGReg holding a corresponding advantage on the Reacher joint angles and the Push-T block pose. The Forward-only ablation is markedly poorer, particularly under linear probes, but is rarely fully degenerate---some residual variance survives even in a partially collapsed encoder. Both regularized representations therefore make the underlying physical state recoverable. 

\paragraph{Latent geometry.} We next turn to the geometry that organizes them. \cref{fig:environment_embeddings} reports, for each environment, the PCA spectrum of held-out embeddings (top), the distribution of a representative ground-truth quantity in physical state space (middle), and the embeddings projected onto a 2- or 3-dimensional PC subspace (bottom). In every environment, the spectrum drops sharply, confirming that the embeddings concentrate on an effectively low-dimensional subspace of the \(192\)-dimensional ambient latent space. The corresponding SIGReg visualizations are provided in \cref{app:sigreg_geometry}; while SIGReg also prevents collapse, its latent geometry appears less consistently organized by the controllable coordinates.

\begin{figure}[tb]
    \centering
    \includegraphics[width=0.95\linewidth]{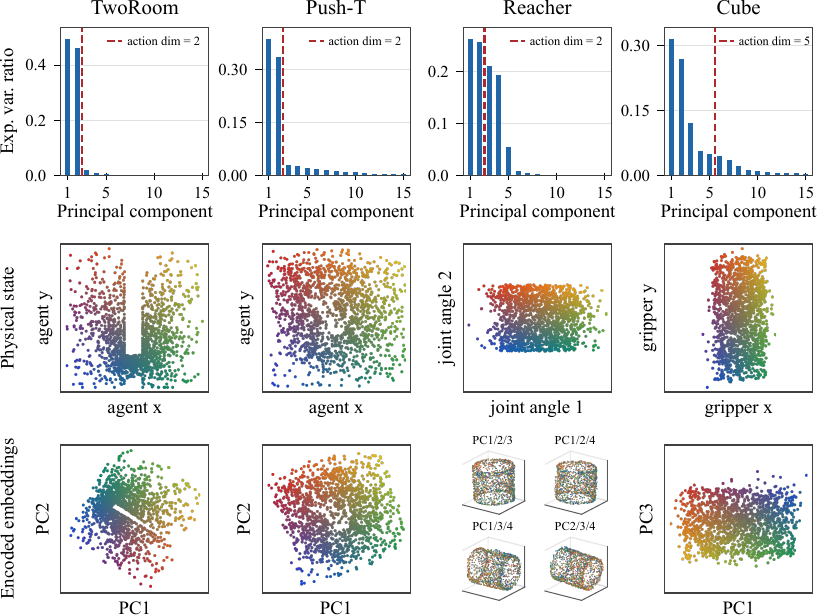}
    \caption{\textbf{Latent geometry of SMWM embeddings.} For each environment we show the PCA spectrum of held-out embeddings (top), the distribution of a representative ground-truth quantity in physical state space (middle), and the embeddings projected onto a 2- or 3-dimensional PC subspace, color-coded by the same physical quantity (bottom). The dashed red lines mark the action dimension. Across all four environments, the embeddings concentrate on a low-dimensional subspace whose geometry mirrors the controllable state space: Cartesian coordinates (TwoRoom and Push-T agent, Cube gripper) appear as approximately linear directions, while the periodic joint angles in Reacher are encoded as a flat 2-torus, visible as a cylindrical structure in any 3D PC projection.}
    \label{fig:environment_embeddings}
\end{figure}

In TwoRoom, the encoder behaves much as in the dot world: \(\mathrm{PC}_1\) and \(\mathrm{PC}_2\) recover a spatially faithful Cartesian map of the agent, and the wall—a region unreachable to the agent—leaves an empty band in latent space, providing further evidence that the recovered representation is spatially faithful. In Push-T, the agent position is again encoded in \((\mathrm{PC}_1, \mathrm{PC}_2)\), but the spectrum has heavier tails: the embedding must additionally accommodate the position and orientation of the manipulable block (cf.~\cref{tab:physical_quantity_probe_r2}). Reacher reveals a richer geometry. Its state space is the torus \(\mathbb{T}^2\) of two periodic angles \((\theta_1, \theta_2)\), which cannot be continuously encoded into \(\mathbb{R}^2\) due to a discontinuity at \(\theta = 0 \sim 2\pi\). A natural representation is the product-of-circles embedding $(\theta_1,\theta_2) \mapsto (\cos\theta_1,\sin\theta_1,\cos\theta_2,\sin\theta_2)$, which lifts \(\mathbb{T}^2\) into \(\mathbb{R}^4\). Its coordinate 3D projections are circular cylinders, which is exactly the signature we observe in the encoded embeddings: the spectrum exposes four appreciable components, and every 3D PC subspace in \cref{fig:environment_embeddings} (bottom) shows a clear cylindrical shape. The two angles, however, spread across all four components, which is consistent with arbitrary linear mixing of the product-of-circles coordinates. Finally, in Cube the gripper position is recovered in \((\mathrm{PC}_1, \mathrm{PC}_3)\); \(\mathrm{PC}_2\) and the heavier tail of the spectrum carry the cube pose and gripper opening (cf.~\cref{tab:physical_quantity_probe_r2}).

\paragraph{Takeaway.}
Across all four environments, SMWM recovers compact latent spaces whose dimension and geometry mirror controllable structure: Cartesian degrees of freedom form approximately linear directions, while periodic variables form circles. The structure carries over to 2D and 3D control, with the largest planning gain over SIGReg appearing on the 3D environment.

\section{Discussion}\label{sec:discussion}

We showed that inverse dynamics regularization addresses two coupled issues simultaneously. First, as an anti-collapse mechanism, a single-step inverse dynamics head suffices to stabilize end-to-end training of a JEPA-style latent world model from pixels, while introducing only one hyperparameter and imposing no distributional prior on the embedding space. Second, the same mechanism biases the encoder toward a compact, action-sufficient representation due to the joint pressure from the forward and inverse dynamics loss terms. Empirically, this pressure yields representations that track the controllable degrees of freedom of the environment, preserve both Cartesian and periodic state geometry, and filter out uncontrollable distractors, although no term in the objective explicitly enforces this structure. The learned representation further supports latent planning that matches or outperforms a strong SIGReg-regularized baseline across four control tasks, with the largest margin on the most complex 3D manipulation environment with a five-dimensional action space. Taken together, these results support the view of the sensorimotor world model as an effective and simple instantiation of a ``perception for action'' bias in latent world-model learning. 

\paragraph{Limitations and future work.} Our analysis assumes that actions are recoverable from consecutive observations, which may fail when distinct actions induce the same visible change. Relatedly, because the encoder currently maps individual frames to latent states, it will not capture quantities that are not identifiable from a single frame, such as velocities, even when these quantities would improve forward prediction; incorporating short observation histories into the encoder or predictor is a natural extension. Furthermore, the action-aligned bias, while desirable, may discard information needed for downstream tasks that depend on aspects of the environment irrelevant to the training actions (we consider this a feature rather than a bug). In addition, biased behavior policies may create action-correlated but uncontrollable distractors, which could be captured by the encoder. Finally, our planning experiments inherit the usual limitations of offline world models: limited dataset coverage constrains the reachable plans, long-horizon rollouts remain vulnerable to compounding model error, and our empirical validation is restricted to moderate-scale simulated control tasks. Future work should study more robust data regimes, multi-frame, and multi-step inverse objectives. While we use inverse loss only as a regularizer, inverse models themselves could be integrated into latent or hierarchical planners to reduce reliance on long autoregressive forward rollouts.

\bibliography{references}

@article{ha2018world,
  title={World models},
  author={Ha, David and Schmidhuber, J{\"u}rgen},
  journal={arXiv preprint arXiv:1803.10122},
  year={2018}
}

@inproceedings{hafner2019dream,
  title={Dream to control: Learning behaviors by latent imagination},
  author={Hafner, Danijar and Lillicrap, Timothy and Ba, Jimmy and Norouzi, Mohammad},
  booktitle={ICLR},
  year={2020}
}

@inproceedings{hafner2020mastering,
  title={Mastering {Atari} with discrete world models},
  author={Hafner, Danijar and Lillicrap, Timothy and Norouzi, Mohammad and Ba, Jimmy},
  booktitle={ICLR},
  year={2021}
}

@article{hafner2023mastering,
  title={Mastering diverse domains through world models},
  author={Hafner, Danijar and Pasukonis, Jurgis and Ba, Jimmy and Lillicrap, Timothy},
  journal={arXiv preprint arXiv:2301.04104},
  year={2023}
}

@inproceedings{hafner2019learning,
  title={Learning latent dynamics for planning from pixels},
  author={Hafner, Danijar and Lillicrap, Timothy and Fischer, Ian and Villegas, Ruben and Ha, David and Lee, Honglak and Davidson, James},
  booktitle={ICML},
  year={2019}
}

@inproceedings{hansen2022temporal,
  title={Temporal difference learning for model predictive control},
  author={Hansen, Nicklas and Wang, Xiaolong and Su, Hao},
  booktitle={ICML},
  year={2022}
}

@article{hansen2024tdmpc2,
  title={TD-MPC2: Scalable, robust world models for continuous control},
  author={Hansen, Nicklas and Su, Hao and Wang, Xiaolong},
  journal={ICLR},
  year={2024}
}

@inproceedings{grill2020bootstrap,
  title={Bootstrap your own latent: A new approach to self-supervised learning},
  author={Grill, Jean-Bastien and Strub, Florian and Altch{\'e}, Florent and Tallec, Corentin and Richemond, Pierre and Buchatskaya, Elena and Doersch, Carl and Avila Pires, Bernardo and Guo, Zhaohan and Gheshlaghi Azar, Mohammad and others},
  booktitle={NeurIPS},
  year={2020}
}

@inproceedings{oord2018representation,
  title={Representation learning with contrastive predictive coding},
  author={van den Oord, Aaron and Li, Yazhe and Vinyals, Oriol},
  booktitle={arXiv preprint arXiv:1807.03748},
  year={2018}
}

@inproceedings{zbontar2021barlow,
  title={Barlow twins: Self-supervised learning via redundancy reduction},
  author={Zbontar, Jure and Jing, Li and Misra, Ishan and LeCun, Yann and Deny, St{\'e}phane},
  booktitle={ICML},
  year={2021}
}

@inproceedings{bardes2022vicreg,
  title={VICReg: Variance-invariance-covariance regularization for self-supervised learning},
  author={Bardes, Adrien and Ponce, Jean and LeCun, Yann},
  booktitle={ICLR},
  year={2022}
}

@inproceedings{chen2020simple,
  title={A simple framework for contrastive learning of visual representations},
  author={Chen, Ting and Kornblith, Simon and Norouzi, Mohammad and Hinton, Geoffrey},
  booktitle={ICML},
  year={2020}
}

@inproceedings{chen2021exploring,
  title={Exploring simple siamese representation learning},
  author={Chen, Xinlei and He, Kaiming},
  booktitle={CVPR},
  year={2021}
}

@inproceedings{pathak2017curiosity,
  title={Curiosity-driven exploration by self-supervised prediction},
  author={Pathak, Deepak and Agrawal, Pulkit and Efros, Alexei A and Darrell, Trevor},
  booktitle={ICML},
  year={2017}
}

@article{
lamb2023guaranteed,
title={Guaranteed Discovery of Control-Endogenous Latent States with Multi-Step Inverse Models},
author={Alex Lamb and Riashat Islam and Yonathan Efroni and Aniket Rajiv Didolkar and Dipendra Misra and Dylan J Foster and Lekan P Molu and Rajan Chari and Akshay Krishnamurthy and John Langford},
journal={Transactions on Machine Learning Research},
issn={2835-8856},
year={2023},
url={https://openreview.net/forum?id=TNocbXm5MZ},
note={}
}

@article{islam2022agent,
  title={Agent-controller representations: Principled offline {RL} with rich exogenous information},
  author={Islam, Riashat and Tomar, Manan and Lamb, Alex and Efroni, Yonathan and Zang, Hongyu and Didolkar, Aniket and Misra, Dipendra and Li, Xin and van Seijen, Harm and Tachet des Combes, Remi and Langford, John},
  journal={arXiv preprint arXiv:2211.00164},
  year={2022}
}

@article{lecun2022path,
  title={A path towards autonomous machine intelligence},
  author={LeCun, Yann},
  journal={OpenReview},
  year={2022}
}

@inproceedings{assran2023self,
  title={Self-supervised learning from images with a joint-embedding predictive architecture},
  author={Assran, Mahmoud and Duval, Quentin and Misra, Ishan and Bojanowski, Piotr and Vincent, Pascal and Rabbat, Michael and LeCun, Yann and Ballas, Nicolas},
  booktitle={CVPR},
  year={2023}
}

@article{bardes2024vjepa,
  title={V-JEPA: Latent video prediction for visual representation learning},
  author={Bardes, Adrien and Garrido, Quentin and Ponce, Jean and Chen, Xinlei and Rabbat, Michael and LeCun, Yann and Assran, Mahmoud and Ballas, Nicolas},
  journal={arXiv preprint arXiv:2402.04252},
  year={2024}
}

@inproceedings{ferns2004metrics,
  title={Metrics for finite {Markov} decision processes},
  author={Ferns, Norm and Panangaden, Prakash and Precup, Doina},
  booktitle={UAI},
  year={2004}
}

@inproceedings{castro2020scalable,
  title={Scalable methods for computing state similarity in deterministic {Markov} decision processes},
  author={Castro, Pablo Samuel},
  booktitle={AAAI},
  year={2020}
}

@inproceedings{schwarzer2021dataefficient,
  title={Data-efficient reinforcement learning with self-predictive representations},
  author={Schwarzer, Max and Anand, Ankesh and Goel, Rishab and Hjelm, R Devon and Courville, Aaron and Bachman, Philip},
  booktitle={ICLR},
  year={2021}
}

@inproceedings{
    cui2024dynamo,
    title={DynaMo: In-Domain Dynamics Pretraining for Visuo-Motor Control},
    author={Zichen Jeff Cui and Hengkai Pan and Aadhithya Iyer and Siddhant Haldar and Lerrel Pinto},
    booktitle={The Thirty-eighth Annual Conference on Neural Information Processing Systems},
    year={2024},
    url={https://openreview.net/forum?id=vUrOuc6NR3}
}

@inproceedings{laskin2020curl,
  title={Curl: Contrastive unsupervised representations for reinforcement learning},
  author={Laskin, Michael and Srinivas, Aravind and Abbeel, Pieter},
  booktitle={International conference on machine learning},
  pages={5639--5650},
  year={2020},
  organization={PMLR}
}

@inproceedings{yarats2021reinforcement,
  title={Reinforcement learning with prototypical representations},
  author={Yarats, Denis and Fergus, Rob and Lazaric, Alessandro and Pinto, Lerrel},
  booktitle={International Conference on Machine Learning},
  pages={11920--11931},
  year={2021},
  organization={PMLR}
}

@article{kostrikov2020image,
  title={Image augmentation is all you need: Regularizing deep reinforcement learning from pixels},
  author={Kostrikov, Ilya and Yarats, Denis and Fergus, Rob},
  journal={arXiv preprint arXiv:2004.13649},
  year={2020}
}

@inproceedings{
    schmidt2024learning,
    title={Learning to Act without Actions},
    author={Dominik Schmidt and Minqi Jiang},
    booktitle={The Twelfth International Conference on Learning Representations},
    year={2024},
    url={https://openreview.net/forum?id=rvUq3cxpDF}
}

@article{rubinstein1997optimization,
  title={Optimization of computer simulation models with rare events},
  author={Rubinstein, Reuven Y},
  journal={European Journal of Operational Research},
  volume={99},
  number={1},
  pages={89--112},
  year={1997},
  publisher={Elsevier}
}

@article{maes2026stable,
  title={stable-worldmodel-v1: Reproducible world modeling research and evaluation},
  author={Maes, Lucas and Lidec, Quentin Le and Haramati, Dan and Massaudi, Nassim and Scieur, Damien and LeCun, Yann and Balestriero, Randall},
  journal={arXiv preprint arXiv:2602.08968},
  year={2026}
}

@inproceedings{sobal2025pldm,
  title         = {Learning from Reward-Free Offline Data: A Case for Planning with Latent Dynamics Models},
  author        = {Sobal, Vlad and Zhang, Wancong and Cho, Kyunghyun and Balestriero, Randall and Rudner, Tim G. J. and LeCun, Yann},
  booktitle     = {Advances in Neural Information Processing Systems 38 (NeurIPS 2025)},
  year          = {2025},
  url           = {https://neurips.cc/virtual/2025/poster/116649}
}

@inproceedings{
sobal2025stresstesting,
title={Stress-Testing Offline Reward-Free Reinforcement Learning: A Case for Planning with Latent Dynamics Models},
author={Vlad Sobal and Wancong Zhang and Kyunghyun Cho and Randall Balestriero and Tim G. J. Rudner and Yann LeCun},
booktitle={7th Robot Learning Workshop: Towards Robots with Human-Level Abilities},
year={2025},
url={https://openreview.net/forum?id=jON7H6A9UU}
}

@inproceedings{
park2025ogbench,
title={{OGB}ench: Benchmarking Offline Goal-Conditioned {RL}},
author={Seohong Park and Kevin Frans and Benjamin Eysenbach and Sergey Levine},
booktitle={The Thirteenth International Conference on Learning Representations},
year={2025},
url={https://openreview.net/forum?id=M992mjgKzI}
}

@article{tassa2018deepmind,
  title={Deepmind control suite},
  author={Tassa, Yuval and Doron, Yotam and Muldal, Alistair and Erez, Tom and Li, Yazhe and Casas, Diego de Las and Budden, David and Abdolmaleki, Abbas and Merel, Josh and Lefrancq, Andrew and others},
  journal={arXiv preprint arXiv:1801.00690},
  year={2018}
}

@inproceedings{haarnoja2018soft,
  title={Soft actor-critic: Off-policy maximum entropy deep reinforcement learning with a stochastic actor},
  author={Haarnoja, Tuomas and Zhou, Aurick and Abbeel, Pieter and Levine, Sergey},
  booktitle={International conference on machine learning},
  pages={1861--1870},
  year={2018},
  organization={PMLR}
}

@article{maes2026leworldmodel,
  title         = {{LeWorldModel}: Stable End-to-End Joint-Embedding Predictive Architecture from Pixels},
  author        = {Maes, Lucas and Le Lidec, Quentin and Scieur, Damien and LeCun, Yann and Balestriero, Randall},
  journal       = {arXiv preprint arXiv:2603.19312},
  year          = {2026},
  url           = {https://arxiv.org/abs/2603.19312}
}

@article{balestriero2025lejepa,
  title         = {{LeJEPA}: Provable and Scalable Self-Supervised Learning Without the Heuristics},
  author        = {Balestriero, Randall and LeCun, Yann},
  journal       = {arXiv preprint arXiv:2511.08544},
  year          = {2025},
  url           = {https://arxiv.org/abs/2511.08544}
}

@inproceedings{zhou2025dinowm,
  title         = {{DINO-WM}: World Models on Pre-trained Visual Features Enable Zero-Shot Planning},
  author        = {Zhou, Gaoyue and Pan, Hengkai and LeCun, Yann and Pinto, Lerrel},
  booktitle     = {Proceedings of the 42nd International Conference on Machine Learning (ICML 2025)},
  series        = {Proceedings of Machine Learning Research},
  volume        = {267},
  pages         = {79115--79135},
  year          = {2025},
  publisher     = {PMLR},
  url           = {https://proceedings.mlr.press/v267/zhou25t.html}
}

@article{assran2025vjepa2,
  title         = {{V-JEPA 2}: Self-Supervised Video Models Enable Understanding, Prediction and Planning},
  author        = {Assran, Mahmoud and Bardes, Adrien and Fan, David and Garrido, Quentin and Howes, Russell and Komeili, Mojtaba and Muckley, Matthew and Rizvi, Ammar and Roberts, Claire and Sinha, Koustuv and Zholus, Artem and Arnaud, Sergio and Gejji, Abha and Martin, Ada and Robert Hogan, Francois and Dugas, Daniel and Bojanowski, Piotr and Khalidov, Vasil and Labatut, Patrick and Massa, Francisco and Szafraniec, Marc and Krishnakumar, Kapil and Li, Yong and Ma, Xiaodong and Chandar, Sarath and Meier, Franziska and LeCun, Yann and Rabbat, Michael and Ballas, Nicolas},
  journal       = {arXiv preprint arXiv:2506.09985},
  year          = {2025},
  url           = {https://arxiv.org/abs/2506.09985}
}

@article{han2026wam,
  title         = {Enhancing Policy Learning with World-Action Model},
  author        = {Han, Yuci and Yilmaz, Alper},
  journal       = {arXiv preprint arXiv:2603.28955},
  year          = {2026},
  url           = {https://arxiv.org/abs/2603.28955}
}

@article{terver2026ebjepa,
  title         = {A Lightweight Library for Energy-Based Joint-Embedding Predictive Architectures},
  author        = {Terver, Basile and Balestriero, Randall and Dervishi, Megi and Fan, David and Garrido, Quentin and Nagarajan, Tushar and Sinha, Koustuv and Zhang, Wancong and Rabbat, Mike and LeCun, Yann and Bar, Amir},
  journal       = {arXiv preprint arXiv:2602.03604},
  year          = {2026},
  url           = {https://arxiv.org/abs/2602.03604}
}

@inproceedings{
yu2025auxjepa,
title={Why and How Auxiliary Tasks Improve {JEPA} Representations},
author={Jiacan Yu and Siyi Chen and Mingrui Liu and Nono Horiuchi and Vladimir Braverman and Zicheng Xu and Dan Haramati and Randall Balestriero},
booktitle={UniReps: 3rd Edition of the Workshop on Unifying Representations in Neural Models},
year={2025},
url={https://openreview.net/forum?id=ZVx4SdKhlc}
}

@inproceedings{brandfonbrener2023inverse,
  title         = {Inverse Dynamics Pretraining Learns Good Representations for Multitask Imitation},
  author        = {Brandfonbrener, David and Nachum, Ofir and Bruna, Joan},
  booktitle     = {Advances in Neural Information Processing Systems 36 (NeurIPS 2023)},
  year          = {2023},
  url           = {https://proceedings.neurips.cc/paper_files/paper/2023/hash/d36dfcdb14473a8526111c221660f2ab-Abstract-Conference.html}
}

@article{levine2024multistep,
  title         = {Multistep Inverse Is Not All You Need},
  author        = {Levine, Alexander and Stone, Peter and Zhang, Amy},
  journal       = {Reinforcement Learning Journal},
  volume        = {2},
  pages         = {884--925},
  year          = {2024},
  note          = {Presented at the Reinforcement Learning Conference (RLC 2024)},
  url           = {https://rlj.cs.umass.edu/2024/papers/Paper117.html}
}

@inproceedings{efroni2022provable,
  title         = {Provably Filtering Exogenous Distractors using Multistep Inverse Dynamics},
  author        = {Efroni, Yonathan and Misra, Dipendra and Krishnamurthy, Akshay and Agarwal, Alekh and Langford, John},
  booktitle     = {International Conference on Learning Representations (ICLR 2022)},
  year          = {2022},
  url           = {https://openreview.net/forum?id=RQLLzMCefQu}
}

@inproceedings{mhammedi2023musik,
  title         = {Representation Learning with Multi-Step Inverse Kinematics: An Efficient and Optimal Approach to Rich-Observation {RL}},
  author        = {Mhammedi, Zakaria and Foster, Dylan J. and Rakhlin, Alexander},
  booktitle     = {Proceedings of the 40th International Conference on Machine Learning (ICML 2023)},
  series        = {Proceedings of Machine Learning Research},
  volume        = {202},
  pages         = {24659--24700},
  year          = {2023},
  publisher     = {PMLR},
  note          = {Oral presentation},
  url           = {https://proceedings.mlr.press/v202/mhammedi23a.html}
}

@inproceedings{seo2022apv,
  title         = {Reinforcement Learning with Action-Free Pre-Training from Videos},
  author        = {Seo, Younggyo and Lee, Kimin and James, Stephen L. and Abbeel, Pieter},
  booktitle     = {Proceedings of the 39th International Conference on Machine Learning (ICML 2022)},
  series        = {Proceedings of Machine Learning Research},
  volume        = {162},
  pages         = {19561--19579},
  year          = {2022},
  publisher     = {PMLR},
  url           = {https://proceedings.mlr.press/v162/seo22a.html}
}

@article{Goodale-Milner,
  author  = {Goodale, Melvyn A. and Milner, A. David},
  title   = {Separate Visual Pathways for Perception and Action},
  journal = {Trends in Neurosciences},
  volume  = {15},
  number  = {1},
  pages   = {20--25},
  year    = {1992}
}

@article{OReganNoe2001,
  author  = {O'Regan, J. Kevin and No{\"e}, Alva},
  title   = {A Sensorimotor Account of Vision and Visual Consciousness},
  journal = {Behavioral and Brain Sciences},
  volume  = {24},
  number  = {5},
  pages   = {939--1031},
  year    = {2001}
}

@incollection{Prinz1990,
  author    = {Prinz, Wolfgang},
  title     = {A Common Coding Approach to Perception and Action},
  booktitle = {Relationships Between Perception and Action: Current Approaches},
  editor    = {Neumann, Odmar and Prinz, Wolfgang},
  publisher = {Springer},
  address   = {Berlin},
  pages     = {167--201},
  year      = {1990}
}

@book{VarelaThompsonRosch1991,
  author    = {Varela, Francisco J. and Rosch, Eleanor and Thompson, Evan },
  title     = {The Embodied Mind: Cognitive Science and Human Experience},
  publisher = {MIT Press},
  address   = {Cambridge, MA},
  year      = {1991}
}

@book{Gibson1979,
	author = {James J. Gibson},
	publisher = {Houghton Mifflin},
	title = {The Ecological Approach to Visual Perception},
	year = {1979}
}

@book{Uexkull1934,
  author    = {von Uexk{\"u}ll, Jakob},
  title     = {Streifz{\"u}ge durch die Umwelten von Tieren und Menschen: Ein Bilderbuch unsichtbarer Welten},
  series    = {Verst{\"a}ndliche Wissenschaft},
  volume    = {21},
  publisher = {Verlag von Julius Springer},
  address   = {Berlin},
  year      = {1934}
}

@article{Scholkopfetal21,
  title = {Toward Causal Representation Learning},
  journal = {Proceedings of the IEEE},
  volume = {109},
  number = {5},
  pages = {612--634},
  year = {2021},
  author = {Sch{\"o}lkopf, B. and Locatello, F. and Bauer, S. and Ke, N. R. and Kalchbrenner, N. and Goyal, A. and Bengio, Y.}
}

@InProceedings{Rubensteinetal17,
  title = {Causal Consistency of Structural Equation Models},
  author = {Rubenstein*, P. K. and Weichwald*, S. and Bongers, S. and Mooij, J. M. and Janzing, D. and Grosse-Wentrup, M. and Sch{\"o}lkopf, B.},
  booktitle = {Proceedings of the Thirty-Third Conference on Uncertainty in Artificial Intelligence (UAI)},
  editors = {Gal Elidan, Kristian Kersting, and Alexander T. Ihler},
  year = {2017},
  url = {http://auai.org/uai2017/proceedings/papers/11.pdf}
}

@book{Hertz1899,
  author    = {Heinrich Hertz},
  title     = {The Principles of Mechanics Presented in a New Form},
  year      = {1899},
  publisher = {Macmillan},
  address   = {London}
}

@article{DBLP:journals/corr/abs-1911-10500,
  author       = {Bernhard Sch{\"{o}}lkopf},
  title        = {Causality for Machine Learning},
  url          = {http://arxiv.org/abs/1911.10500},
  eprinttype   = {arXiv},
  eprint       = {1911.10500},
  year    = {2019},
  note = {Published in: Probabilistic and Causal Inference: The Works of Judea Pearl}
}

@InProceedings{pmlr-v235-park24c,
   title = {The Linear Representation Hypothesis and the Geometry of Large Language Models},
   author = {Park, Kiho and Choe, Yo Joong and Veitch, Victor},
   booktitle = {Proceedings of the 41st International Conference on Machine Learning},
   pages = {39643--39666},
   year = {2024},
   volume = {235},
   series = {Proceedings of Machine Learning Research},
   publisher = {PMLR}
 }

@inproceedings{RavindranBarto,
  author    = {Balaraman Ravindran and Andrew G. Barto},
  title     = {SMDP Homomorphisms: An Algebraic Approach to Abstraction in Semi-Markov Decision Processes},
  booktitle = {International Joint Conference on Artificial Intelligence (IJCAI)},
  pages      = {1011--1016},
  year       = {2003}
}

@conference{Keurtietal23,
  title = {Homomorphism {A}uto{E}ncoder --- Learning Group Structured Representations from Observed Transitions},
  author = {Keurti, H. and Pan, H.-R. and Besserve, M. and Grewe, B. F. and Sch{\"o}lkopf, B.},
  booktitle = {Proceedings of the 40th International Conference on Machine Learning},
  volume = {202},
  pages = {16190--16215},
  series = {Proceedings of Machine Learning Research},
  editors = {A. Krause, E. Brunskill, K. Cho, B. Engelhardt, S. Sabato and J. Scarlett},
  publisher = {PMLR},
  year = {2023},
  event_place = {Honolulu, Hawaii, USA},
  url = {https://proceedings.mlr.press/v202/keurti23a.html},
  month_numeric = {7}
}
\bibliographystyle{unsrtnat}


\appendix

\section{Implementation details}\label{sec:implementation}

\subsection{Training objective}
\label{app:training_objective}

\cref{alg:loss} gives PyTorch-style pseudocode for the mini-batch objective used to train SMWM. The encoder receives gradients from both the forward prediction loss and the inverse dynamics loss, which is what prevents representation collapse.

\begin{algorithm}[tbh]
\caption{Mini-batch loss for SMWM training.}
\label{alg:loss}
\begin{lstlisting}[style=pytorch]
def loss(o_t, a_t, o_tp1, lambda_inv=10.0):
    """
    o_t, o_tp1: (B, C, H, W)  consecutive pixel observations
    a_t:        (B, A)        action taken between them
    lambda_inv: (float)       inverse dynamics loss weight
    """
    z_t   = encoder(o_t)                    # (B, D)
    z_tp1 = encoder(o_tp1)                  # (B, D)

    z_hat = fwd_model(z_t, a_t)             # predict next embedding
    a_hat = inv_model(z_t, z_tp1)           # predict intervening action

    loss_fwd = F.mse_loss(z_hat, z_tp1)     # forward dynamics loss
    loss_inv = F.mse_loss(a_hat, a_t)       # inverse dynamics loss

    return loss_fwd + lambda_inv * loss_inv
\end{lstlisting}
\end{algorithm}

\subsection{Architecture}\label{sec:architecture}

We adopt the encoder and forward model of LeWorldModel~\citep{maes2026leworldmodel} essentially unchanged, and add a small inverse head \(h_\psi\). Inputs are \(224\times 224\) RGB images.

\paragraph{Encoder \(f_\theta\).} A Vision Transformer Tiny (ViT-Tiny) with patch size \(14\), instantiated from the Hugging Face library. The final \textsc{cls} token is projected to \(z \in \mathbb{R}^d\) with \(d = 192\); the encoder has \(\approx 5\)M parameters.

\paragraph{Forward model \(g_\phi\).} A ViT-S backbone with learned positional embeddings and causal masking over the observation history, matching the LeWM predictor; the action \(a_t\) is injected via AdaLN-zero. However, in contrast to LeWM, we set the history length to \(1\) on \emph{every} environment, so \(g_\phi\) predicts \(\hat z_{t+1}\) from a single \((z_t, a_t)\) pair, as described in \cref{sec:method}. LeWM instead uses history \(3\) on Push-T and OGBench-Cube. The forward model has \(\approx 10\)M parameters.

\paragraph{Inverse model \(h_\psi\).} Our only architectural addition. A 2-layer MLP of width \(256\) with ReLU activations, mapping \([z_t;\, z_{t+1}] \in \mathbb{R}^{2d}\) to \(\hat a_t \in \mathbb{R}^m\), without output activation. Approximately \(0.1\)M parameters.

\subsection{Optimization details}\label{sec:optimization}

All world models are trained end-to-end with the objective in \cref{eq:total-loss}. We use the same optimization settings for SMWM, the forward-only ablation, and the SIGReg baseline; only the regularization term is changed. For SMWM, the inverse-dynamics weight $\lambda$ is environment-specific. The values are listed in \cref{tab:lambda_values}.

\begin{table}[tbh]
    \centering
    \caption{\textbf{Inverse-dynamics weights.} Values of $\lambda$ in $\mathcal{L}_{\text{fwd}}+\lambda\mathcal{L}_{\text{inv}}$.}
    \label{tab:lambda_values}
    \begin{tabular}{lc}
        \toprule
        Environment & $\lambda$ \\
        \midrule
        TwoRoom & $0.1$ \\
        Reacher & $5$ \\
        Push-T & $30$ \\
        OGBench-Cube & $1$ \\
        \bottomrule
    \end{tabular}
\end{table}

The remaining training hyperparameters are shared across environments unless explicitly noted in \cref{tab:optimization_hyperparameters}. We train on two-frame snippets with a single-step forward target. Since datasets are loaded with frameskip $5$, the action encoder and inverse head receive the concatenated action over the skipped interval. The SIGReg comparison uses the same optimizer and architecture, with $\lambda=0$ and SIGReg weight $0.09$.

\begin{table}[tbh]
    \centering
    \caption{\textbf{Optimization hyperparameters.}}
    \label{tab:optimization_hyperparameters}
    \small
    \begin{tabular}{lc}
        \toprule
        Hyperparameter & Value \\
        \midrule
        Optimizer & AdamW \\
        Learning rate & $10^{-4}$ \\
        Weight decay & $10^{-3}$ \\
        Batch size & $256$ \\
        Training epochs & $10$ \\
        Gradient clipping & $1.0$ \\
        \bottomrule
    \end{tabular}
\end{table}

\paragraph{Choosing $\lambda$.}
\Cref{fig:lambda_sweep_planning} shows that $\lambda$ requires some environment-level tuning. For example, $\lambda=0.1$ is optimal for TwoRoom but causes Reacher to collapse to near-random planning performance. The final values in \cref{tab:lambda_values} were chosen from this broad sweep. The sweep contains one trained model per $\lambda$ value, so the shaded bands show binomial standard error over the 100 evaluation episodes rather than variation across training seeds.

\begin{figure}[tbh]
    \centering
    \includegraphics[width=0.95\linewidth]{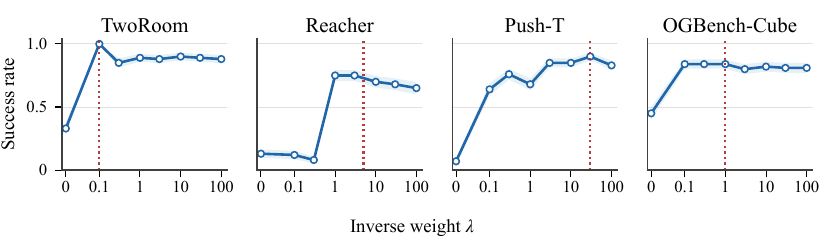}
    \caption{\textbf{Sensitivity to inverse-dynamics weight.} Goal-conditioned planning success rate as a function of the inverse-dynamics loss weight $\lambda$ at goal offset 25. Each panel corresponds to one environment, and the red dotted line marks the value used for the main paper experiments.}
    \label{fig:lambda_sweep_planning}
\end{figure}

\subsection{Planning protocol}\label{planning_protocol}

For each evaluation episode, the policy receives the current observation \(o_t\) and a goal observation \(o_g\), encodes them as \(z_t=f_\theta(o_t)\) and \(z_g=f_\theta(o_g)\), and optimizes candidate action sequences with CEM under the terminal cost in \cref{eq:planning-cost}. Candidate sequences are rolled out using \(g_\phi\) for \(H\) latent model steps, where each model-step action is a block of five primitive environment actions.

Planning uses the same MPC and CEM hyperparameters in all four environments (\cref{tab:planning_protocol}). We sample 100 start states without replacement from the held-out evaluation dataset, requiring that the same episode contains a goal frame 25 steps later. The environment is reset to the sampled start state, the target is set from the goal row, and success is measured over a budget of 50 primitive environment steps. The sampled tasks are fixed across methods and seeds within each environment.

\begin{table}[tbh]
    \centering
    \caption{\textbf{Planning hyperparameters.}}
    \label{tab:planning_protocol}
    \small
    \begin{tabular}{@{}ll@{}}
        \toprule
        Hyperparameter & Value \\
        \midrule
        Planning horizon \(H\) & 5 model steps \\
        Executed steps \(K\) & 5 model steps \\
        Action block & 5 primitive actions \\
        CEM population & 300 \\
        CEM iterations & 30 \\
        CEM elites & 30 \\
        Initial CEM variance scale & 1.0 \\
        Evaluation tasks & 100 \\
        Goal offset & 25 primitive steps \\
        Evaluation budget & 50 primitive steps \\
        \bottomrule
    \end{tabular}
\end{table}

\subsection{Probing protocol.}
For \cref{tab:physical_quantity_probe_r2}, we freeze \(f_\theta\), encode embeddings \(z=f_\theta(o)\), and train probes on \(25{,}000\) training embeddings per environment and method, reporting held-out \(R^2\) on a disjoint \(5{,}000\)-sample validation split. The linear probe is ridge regression with \(\alpha=10^{-3}\). The nonlinear probe is a two-layer MLP with widths \(256,128\) and ReLU activations. For vector-valued targets, including sine--cosine angle encodings, we report uniformly averaged \(R^2\) over target coordinates.

\subsection{Compute.}
All training and planning runs used a single NVIDIA H100 GPU. Each world model was trained in under 4 hours.

\section{Environments and datasets}\label{app:environments}

Our experiments use the four environments and offline datasets adopted by \citet{maes2026leworldmodel} (\cref{fig:environments}). All models are trained and evaluated through the \texttt{stable-worldmodel} codebase~\citep{maes2026stable}, which standardizes the environment interfaces, dataset loading, and downstream planning procedure.

For every environment, we use an episode-level \(90/10\) train/validation split: \(f_\theta\), \(g_\phi\), and \(h_\psi\) are trained on the training episodes, while all reported metrics---planning (\cref{fig:main-planning,fig:planning_horizon_sweep}), probes (\cref{tab:physical_quantity_probe_r2}), and latent-geometry analyses---are computed on the held-out validation split.

\paragraph{(a) TwoRoom.} TwoRoom~\citep{sobal2025pldm, sobal2025stresstesting} is a 2D point-navigation task with a simple constrained geometry: the agent starts in one room and the goal lies in the other, with a single doorway as the only passage between them. The script moves the agent first toward the doorway and then toward the goal after crossing the wall. The dataset contains 10{,}000 noisy scripted rollouts, averaging 92 steps.

\paragraph{(b) Push-T.}
Push-T is a 2D contact-rich manipulation task in which a dot-shaped agent must push a T-shaped object into a target pose. We use the DINO-WM expert demonstrations~\citep{zhou2025dinowm}, consisting of 20{,}000 episodes with mean length 196.

\paragraph{(c) OGBench-Cube.}
OGBench-Cube~\citep{park2025ogbench} is a tabletop manipulation task where a robot gripper must move a cube to a specified target pose. We use the single-cube setting and the benchmark's scripted dataset: 10{,}000 trajectories, each of length 200.

\paragraph{(d) Reacher.}
Reacher is the two-link arm task from the DeepMind Control Suite~\citep{tassa2018deepmind}. We follow the DINO-WM variant, where success is defined by matching the target joint configuration, not just by end-effector proximity. The offline dataset contains 10{,}000 SAC-generated episodes~\citep{haarnoja2018soft}, each 200 steps long.

\begin{figure}[tbh]
    \centering
    \includegraphics[width=0.95\linewidth]{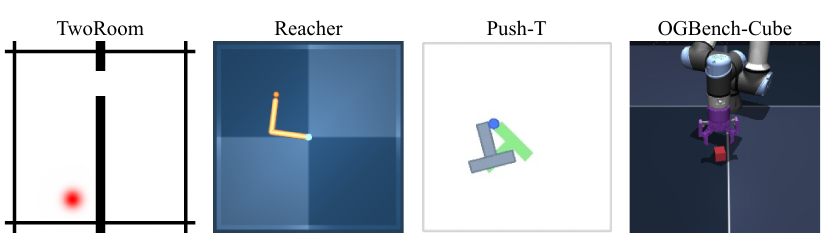}
    \caption{\textbf{Environments.} Four evaluation environments---TwoRoom (2D navigation), Reacher (continuous control), Push-T (2D contact-rich manipulation), and OGBench-Cube (3D tabletop manipulation).}
    \label{fig:environments}
\end{figure}

\section{Planning performance on extended goal offsets} 

\Cref{fig:planning_horizon_sweep} evaluates whether the learned latent dynamics remain useful when the goal is placed farther from the initial state. For each goal offset, the start and goal observations are sampled from the held-out evaluation split, and the planner is given a budget of \(2\times\) the offset in primitive environment steps. Thus the evaluation becomes harder because the model must support longer latent rollouts. The same trained models and MPC/CEM planner from \cref{fig:main-planning} are used throughout.

SMWM is stable across longer horizons on TwoRoom and OGBench-Cube, where SIGReg either degrades sharply or remains consistently lower. On Reacher, the inverse and SIGReg curves stay close over the tested offsets. Push-T is the main failure case: both regularized methods drop as the offset increases.
\begin{figure}[tbh]
    \centering
    \includegraphics[width=0.95\linewidth]{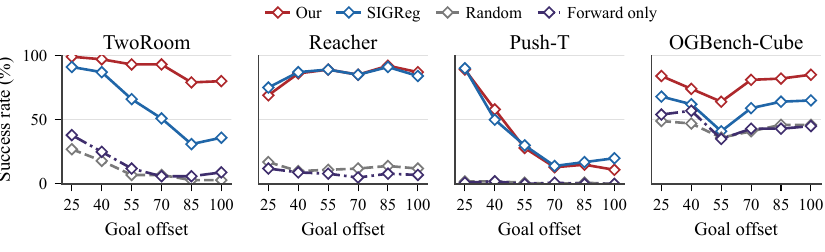}
     \caption{\textbf{Robustness to planning horizon.} Goal-conditioned planning success rate as a function of \emph{goal offset}, the number of environment steps between the initial and goal observations; the planner's evaluation budget is fixed at $2\times$ the goal offset. Methods and environments match \cref{fig:main-planning}. SMWM maintains near-flat success across horizons on TwoRoom and OGBench-Cube, while SIGReg either degrades sharply (TwoRoom) or trails consistently (Cube). The two regularized methods are comparable on Reacher and Push-T.
     }
    \label{fig:planning_horizon_sweep}
\end{figure}

\section{SIGReg latent geometry}\label{app:sigreg_geometry}

\Cref{fig:sigreg_environment_embeddings} repeats the representation summary of \cref{fig:environment_embeddings} for the SIGReg baseline. The same held-out embeddings and PCA protocol are used, with the bottom row showing \((\mathrm{PC}_1,\mathrm{PC}_2)\) for each environment. Unlike SMWM embeddings, the SIGReg spectra do not exhibit a clear elbow. This is expected: the SIGReg regularizer explicitly encourages the embedding distribution to match an isotropic Gaussian, spreading variance across many principal components rather than concentrating it in a compact task-relevant subspace. As a result, although state information remains recoverable by probes, the resulting representation is not low-dimensional, and is less directly interpretable than the SMWM representation.

\begin{figure}[tbh]
    \centering
    \includegraphics[width=0.95\linewidth]{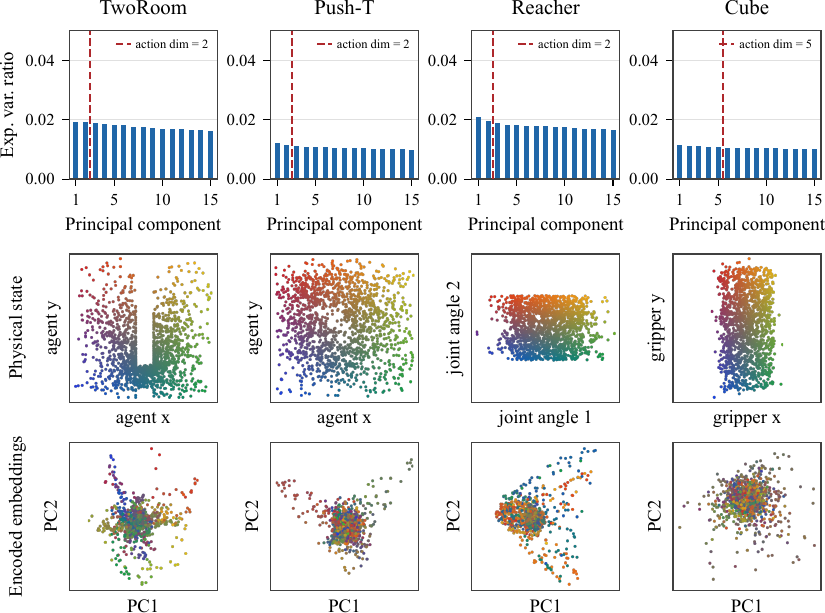}
    \caption{\textbf{Latent geometry of SIGReg embeddings.} For each environment we show the PCA spectrum of held-out embeddings (top), the distribution of a representative ground-truth quantity in physical state space (middle), and the first two principal components of the encoded embeddings, color-coded by the same physical quantity (bottom). The dashed red lines mark the action dimension. Compared with SMWM embeddings in \cref{fig:environment_embeddings}, SIGReg produces less compact and less structured latent geometry: the spectra lack a sharp elbow, and the state-dependent structure is spread across many components rather than appearing in a small interpretable PC subspace.}
    \label{fig:sigreg_environment_embeddings}
\end{figure}

\section{Dot world details}\label{app:dotworld}

\subsection{Environment}\label{app:dotworld:env}

Each observation is a $64{\times}64$ RGB image rendered on a white canvas: each dot is drawn as a filled disc of radius $r = 2$ pixels in a fixed RGB color. Integer dot positions are sampled uniformly in $[m, 64-1-m]^2$ with a margin of $m = r + \Delta_{\max} = 18$, and integer per-step displacements in $[-\Delta_{\max}, \Delta_{\max}]^2$ with $\Delta_{\max} = 16$, so every sampled transition stays inside the canvas. Rejection sampling enforces a minimum center-to-center separation of $2r+1 = 5$ pixels at both endpoints. Each $(o_t, a_t, o_{t+1})$ is generated lazily from a seeded RNG, making the dataset deterministic and reproducible.

The four configurations of \cref{fig:dotworld-multidot} combine three motion types: Independent dots have their own $(\Delta x, \Delta y)$ in the action vector, coupled dot pairs share a single $(\Delta x, \Delta y)$, and random dots move at every step without contributing to the action vector. Concretely, \textbf{Independent} = 2 independent dots ($|a|=4$); \textbf{Coupled} = 1 coupled pair ($|a|=2$); \textbf{Distractor} = 1 independent dot + 1 random dot ($|a|=2$); \textbf{Combined} = 1 coupled pair + 2 independent dots + 1 random dot ($|a|=6$).

\subsection{Architecture and training}\label{app:dotworld:train}

The encoder $f_\theta$ stacks three stride-2 convolutions ($3{\to}32{\to}64{\to}128$ channels, $3{\times}3$ kernels) interleaved with ReLU, followed by a $2{\times}2$ max-pool, a flatten, a linear projection to $\mathbb{R}^d$, and a final $\tanh$ that bounds embeddings in $[-1, 1]^d$ with $d = 64$. Both $g_\phi$ and $h_\psi$ are 2-hidden-layer MLPs of width $256$ with ReLU activations; $g_\phi$ ends in $\tanh$ to match the encoder range, while $h_\psi$ has no output activation. Approximate parameter counts: $f_\theta \approx 220$k, $g_\phi \approx 100$k, $h_\psi \approx 100$k.

All three networks are trained jointly with Adam (lr $= 5{\times}10^{-4}$, default $\beta$, no weight decay, no LR schedule) for 100 epochs at batch size 256 on 250{,}000 transitions, with seed 42 and $\lambda = 10$. Actions are normalized by $\Delta_{\max}$ before being passed to $g_\phi$ and predicted by $h_\psi$, so all action coordinates lie in $[-1, 1]$; we report $\mathcal{L}_{\text{inv}}$ on these normalized actions. All metrics are computed on a held-out evaluation split of $1{,}000$ transitions sampled with a disjoint seed.

\subsection{Collapse without the inverse loss}

To verify that the inverse loss is the active anti-collapse mechanism on dot world, we rerun the single-dot configuration with $\lambda = 0$ and otherwise identical hyperparameters. \Cref{fig:dotworld-collapse} reports the resulting representation using the same probe grid as \cref{fig:dotworld-pca}. Without the inverse loss, the encoder collapses all probe observations to essentially the same embedding. Consequently, there is no state-dependent variation for PCA to expose. The forward loss alone can therefore be minimized by a constant, trivially predictable representation.

\begin{figure}[tbh]
    \centering
    \includegraphics[width=0.95\linewidth]{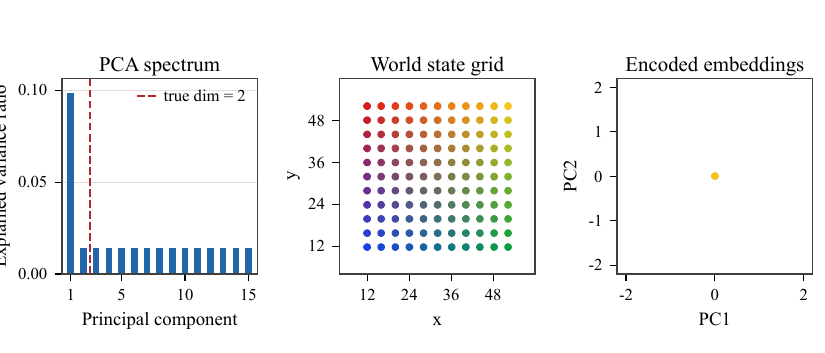}
    \caption{\textbf{Forward-only collapse on dot world.} The single-dot model is trained with $\lambda = 0$ and evaluated with the same probe grid and visualization protocol as \cref{fig:dotworld-pca}. Removing the inverse loss maps the probe states to the same embedding, so the PC projection collapses to a single point at the origin.}
    \label{fig:dotworld-collapse}
\end{figure}

\section{Sprite world reconstruction experiment}
\label{app:spriteworld}

\subsection{Environment and control settings}
\label{app:spriteworld:env}

To make the control-dependent nature of the learned representation visually explicit, we introduce a second toy environment with a single triangular agent moving in a 2D plane. The agent is rendered as an asymmetric sprite on a \(64{\times}64\) RGB canvas. Its state is its pose,
\[
    s = (x,y,\theta),
\]
where \(x,y\) denote position and \(\theta\) denotes orientation. Because the triangle is asymmetric, its orientation is visible from a single frame.

The key intervention in this experiment is to vary the action repertoire while keeping the underlying visual world fixed. In all settings, the agent can move in all three degrees of freedom: \(x\), \(y\), and \(\theta\). The difference is only which changes are exposed to the model as actions. Uncontrolled degrees of freedom are not held fixed; they continue to move randomly.

We consider the following control settings:
\begin{itemize}
    \item \textbf{No control:} the agent translates and rotates randomly, and the action vector is empty.
    \item \textbf{\(x/y\) control:} the action is \(a_t=(\Delta x,\Delta y)\), while rotations remain random and uncontrolled.
    \item \textbf{\(x/y/\theta\) control:} the action is \(a_t=(\Delta x,\Delta y,\Delta\theta)\), so both translation and rotation are controlled.
\end{itemize}

Thus the same triangular agent is observed under different action interfaces. This lets us ask whether the representation changes depending on which aspects of the agent's pose are controllable.

\subsection{Training and reconstruction probe}
\label{app:spriteworld:decoder}

For each control setting, we train the same inverse-dynamics-regularized world model as in the dot-world experiments: a CNN encoder \(f_\theta\), a forward model \(g_\phi\), and an inverse model \(h_\psi\), optimized with the objective in \cref{eq:total-loss}. The architecture and training procedure match \cref{app:dotworld:train}. In the no-control setting, the action dimension is zero, so there is no inverse-dynamics signal and the model reduces to the forward-only objective.

After training, we freeze the encoder and train a decoder as a post-hoc probe,
\[
    \hat{o}_t = d_\xi(f_\theta(o_t)).
\]
The decoder is trained only to visualize what information is already present in the frozen embedding \(z_t=f_\theta(o_t)\); it is not used to train the representation.

\begin{figure}[tbh]
    \centering
    \includegraphics[width=0.95\linewidth]{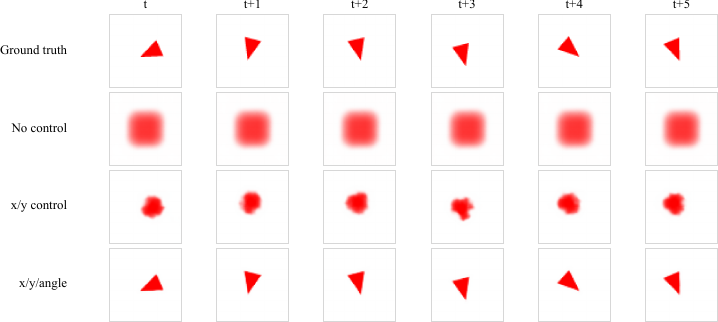}
    \caption{\textbf{Control-dependent reconstruction of a triangular agent.}
    The top row shows the ground-truth trajectory of an asymmetric triangular agent with pose \((x,y,\theta)\). The remaining rows show reconstructions from frozen embeddings learned under different action interfaces. With no control, the representation collapses and the decoder outputs an average occupancy pattern. With \(x/y\) control, the representation preserves position but averages over the uncontrolled orientation, producing a blurred, approximately symmetric blob at the correct location. With full \(x/y/\theta\) control, the representation preserves both position and orientation, yielding sharp oriented reconstructions.}
    \label{fig:sprite-control-reconstruction}
\end{figure}

\subsection{Interpretation}
\label{app:spriteworld:interpretation}

The experiment shows that the model ``sees'' the same visual object differently depending on what the agent can control. When there are no actions, there is no inverse signal that could force the encoder to preserve state information, so the representation collapses. When only translation is controlled, the encoder preserves where the agent is, because this is needed to recover and predict translational effects, but it discards orientation, because rotation is random and action-irrelevant. The decoder therefore averages over possible orientations. When rotation is also controlled, orientation becomes action-relevant, and the encoder preserves the full pose.

This provides a direct visual counterpart to the PCA results in \cref{fig:dotworld-multidot}: the learned representation retains controllable degrees of freedom and averages out uncontrolled ones.

\section{Equivariance and latent action composition}
\label{app:equivariance-composition}

This appendix gives the formal statement behind the compositional claim in \cref{sec:dotworld}. Denote the effect of action \(a\) on the observation \(o\) by \(o\mapsto a(o)\), and write \(g_a(z):=g_\phi(z,a)\). In the main text, the forward prediction loss encourages the approximate equivariance condition
\begin{equation}
    f(a(o)) \approx g_a(f(o)),
\end{equation}
on the support of the training distribution. Here, for clarity, we consider the exact version
\begin{equation}
    \label{eq:equiv-exact}
    f(a(o)) = g_a(f(o))
\end{equation}
for all \(a\in\mathcal A\) and \(o\in\mathcal O\). The encoder thus intertwines the physical intervention with the learned latent intervention, making this a statement of \emph{equivariance}. In physics terminology one might also say that the latent state transforms \emph{covariantly} with the intervention.

We consider sequences of actions, assuming that the actions form a semigroup under composition; i.e., for \(a_1,a_2\in\mathcal A\), the composition \(a_2\circ a_1\) is again in \(\mathcal A\), and composition is associative. We say that \(a\mapsto g_a\) is a homomorphism on the latent manifold \(f(\mathcal O)\) if, for every \(a_1,a_2\in\mathcal A\), we have
\begin{equation}
    g_{a_2\circ a_1}=g_{a_2}\circ g_{a_1}
    \qquad \text{on } f(\mathcal O),
\end{equation}
i.e., for every \(z\in f(\mathcal O)\),
\begin{equation}
    g_{a_2\circ a_1}(z)
    =
    g_{a_2}(g_{a_1}(z)).
\end{equation}

\begin{proof}[Proof of \cref{thm:equivariance-homomorphism}]
Choose any \(a_1,a_2\in\mathcal A\) and \(z\in f(\mathcal O)\). Choose \(o\in\mathcal O\) with \(f(o)=z\). By closure, \(a_3:=a_2\circ a_1\in\mathcal A\).

Applying equivariance three times,
\begin{align*}
    g_{a_2\circ a_1}(z)
        &= g_{a_3}\bigl(f(o)\bigr) \\
        &= f\bigl(a_3(o)\bigr)
            && \text{[equivariance at } (a_3,o)\text{]} \\
        &= f\bigl(a_2(a_1(o))\bigr) \\
        &= g_{a_2}\!\bigl(f(a_1(o))\bigr)
            && \text{[equivariance at } (a_2,a_1(o))\text{]} \\
        &= g_{a_2}\!\bigl(g_{a_1}(f(o))\bigr)
            && \text{[equivariance at } (a_1,o)\text{]} \\
        &= g_{a_2}\!\bigl(g_{a_1}(z)\bigr).
\end{align*}
Thus \(g_{a_2\circ a_1}(z)=g_{a_2}(g_{a_1}(z))\) for every \(z\in f(\mathcal O)\), so \(a\mapsto g_a\) is a homomorphism on \(f(\mathcal O)\).
\end{proof}

One-step equivariance thus forces physical action composition to be represented by composition of the learned latent interventions, at least on the encoded data manifold. In the learned model, this identity should be interpreted approximately and on-support. Note that the assumption of semigroup closure is nontrivial: in practice, it could happen that a multi-step action sequence is not equivalent to any single primitive action.\footnote{This case would correspond to using a free semigroup of action sequences.}\footnote{The homomorphism property established in Theorem~\ref{thm:equivariance-homomorphism} is closely related to the notion of MDP homomorphisms \citep{RavindranBarto}, where abstraction maps preserve the compositional structure of state-action transitions (cf.\ also \cite{Rubensteinetal17,Keurtietal23}). In our setting, however, the homomorphism emerges from approximate equivariance induced by the forward prediction objective, rather than being assumed as part of an abstraction framework. A related discussion can  be found in \cite{DBLP:journals/corr/abs-1911-10500}.}

Equivariance alone admits a wide range of solutions, including two trivial extremes:
\begin{enumerate}[leftmargin=*,itemsep=2pt,topsep=2pt]
    \item \textbf{No encoding.} Take \(\mathcal Z=\mathcal O\), \(f=\mathrm{id}\), and \(g_a=a\). Then both sides of \cref{eq:equiv-exact} equal \(a(o)\). This solution satisfies equivariance but achieves no compression.
    \item \textbf{Collapse.} Take \(\mathcal Z=\{z\}\), \(f\equiv z\), and \(g_a=\mathrm{id}\). Then both sides of \cref{eq:equiv-exact} equal \(z\). This solution satisfies equivariance but discards all information about \(a\).
\end{enumerate}
Useful representations therefore need more than equivariance: the latent dynamics should remain faithful to the physical action. In particular, if an action \(a\in\mathcal A\) changes observations nontrivially in \(\mathcal O\), it should also induce a nontrivial transformation on \(f(\mathcal O)\). A stronger requirement is latent action identifiability: whenever \(z'=g_a(z)\) for \((z,z')\in f(\mathcal O)\times f(\mathcal O)\), the action \(a\) should be recoverable from the latent transition \((z,z')\). This is exactly the constraint imposed by the inverse model, which predicts \(a\) from \((z,z')\).

Beyond a homomorphism into latent transformations under composition, the experiments in \cref{fig:dotworld-rollout} suggest that in our case the homomorphism may approximately go into the additive group, i.e., the learned interventions appear to act approximately as translations,
\begin{equation}
    g_a(z) \approx z+\rho(a),
\end{equation}
with \(\rho(a)\) approximately independent of \(z\). Here \(\rho(a)\) is the latent displacement, or latent effect vector, associated with action \(a\). Substituting this into the composition identity yields
\begin{equation}
    z+\rho(a_2\circ a_1)
    \approx
    g_{a_2}(g_{a_1}(z))
    \approx
    z+\rho(a_1)+\rho(a_2),
\end{equation}
and therefore
\begin{equation}
    \rho(a_2\circ a_1) \approx \rho(a_1)+\rho(a_2).
\end{equation}
For pure displacement actions, where composition is itself approximately additive away from boundaries, this means that the model has learned an approximately additive latent representation of interventions. This stronger property is empirical rather than imposed by the objective.\footnote{It need not hold for contact-rich dynamics, boundary effects, periodic variables, strongly state-dependent action effects, or noncommutative action laws.}

It is intriguing to ask where this translation-like structure may stem from. The inverse-dynamics loss provides a plausible bias toward this regime. Equivariance alone admits degenerate solutions, including collapse; useful representations additionally require the latent action to be identifiable from the transition. The inverse head enforces this by requiring \(a_t\) to be recoverable from \((z_t,z_{t+1})\), where \(z_t=f(o_t)\) and \(z_{t+1}=f(o_{t+1})\). A simple way to satisfy this constraint is to encode the action in the displacement vector
\begin{equation}
    z_{t+1}-z_t \approx \rho(a_t).
\end{equation}
If \(\rho\) is injective on the action support, then the inverse \(h\) can decode the action by applying an approximate inverse of \(\rho\) to this displacement. This mechanism does not enforce additivity, but it renders an additive, translation-like representation a low-complexity solution.

This perspective connects the empirical latent geometry to an interventional version of the linear representation hypothesis \citep{pmlr-v235-park24c}. The claim is stronger than ordinary linear decodability: it is not merely that actions can be decoded from latent states, but that applying an intervention corresponds approximately to moving along an action-dependent direction in latent space,
\begin{equation}
    f(a(o))-f(o) \approx \rho(a).
\end{equation}
Thus the learned representation appears to linearize controllable interventions: the encoder is approximately equivariant to action-induced transformations, and in simple displacement domains the learned latent interventions further specialize to additive translations.

\section{Extended related work}\label{app:extended}

This appendix expands the discussion of \cref{sec:related}. \cref{tab:detailed_comparison} summarizes the discussed methods along the dimensions most relevant to our contribution.

\begin{table*}[t]
\centering
\caption{\textbf{Comparison of related methods.} ``Inv.''\ and ``Fwd.''\ denote inverse and forward dynamics components; ``True action''\ indicates supervision with ground-truth actions; ``WM'' marks methods that deliver a world model usable for planning. The last column lists the primary anti-collapse mechanism. SMWM is the only entry that combines offline reward-free world-model training with inverse dynamics as the sole anti-collapse mechanism.}
\label{tab:detailed_comparison}
\small
\setlength{\tabcolsep}{4pt}
\begin{tabular}{@{}lccccccl@{}}
\toprule
\textbf{Method} & \textbf{Inv.} & \textbf{Fwd.} & \textbf{True action} & \textbf{Reward} & \textbf{Offline} & \textbf{WM} & \textbf{Anti-collapse} \\
\midrule
ICM \citep{pathak2017curiosity}                      & \checkmark & \checkmark & \checkmark & \texttimes & \texttimes & \texttimes  & Inv.\ dynamics\ on $\phi_\theta$ \\
CURL \citep{laskin2020curl}                          & \texttimes & \texttimes & \checkmark & \checkmark & \texttimes & \texttimes  & Contrastive \\
Dreamer \citep{hafner2023mastering}                  & \texttimes & \checkmark & \checkmark & \checkmark & \texttimes & \checkmark  & Recon.\ $+$ KL \\
TD-MPC \citep{hansen2022temporal}                    & \texttimes & \checkmark & \checkmark & \checkmark & \texttimes & \checkmark  & EMA latent target \\
DynaMo \citep{cui2024dynamo}                         & \checkmark & \checkmark & \texttimes & \texttimes & \checkmark & \checkmark  & Joint inv.\ $+$ fwd.\ dyn. \\
PLDM \citep{sobal2025pldm}                           & \checkmark & \checkmark & \checkmark & \texttimes & \checkmark & \checkmark  & VICReg $+$ inv.\ dyn. \\
DINO-WM \citep{zhou2025dinowm}                       & \texttimes & \checkmark & \checkmark & \texttimes & \checkmark & \checkmark  & Frozen DINOv2 \\
LeWorldModel \citep{maes2026leworldmodel}            & \texttimes & \checkmark & \checkmark & \texttimes & \checkmark & \checkmark  & SIGReg \\
WAM \citep{han2026wam}                               & \checkmark & \checkmark & \checkmark & \checkmark & \texttimes & \checkmark  & Recon.\ $+$ KL $+$ inv.\ dyn. \\
EB-JEPA \citep{terver2026ebjepa}                     & \checkmark & \checkmark & \checkmark & \texttimes & \checkmark & \checkmark  & VICReg $+$ inv.\ dyn. \\
\midrule
\textbf{SMWM}                                        & \checkmark & \checkmark & \checkmark & \texttimes & \checkmark & \checkmark  & Inv.\ dynamics\ alone \\
\bottomrule
\end{tabular}
\end{table*}

\subsection{Inverse dynamics for representation learning}

\subsubsection{Intrinsic Curiosity Module (ICM)}
ICM \citep{pathak2017curiosity} is the canonical inverse-dynamics-as-feature-learner construction we extend. Two architectural differences matter for our claim. First, in ICM the encoder receives gradients only from the inverse loss; the forward loss treats \(\phi_\theta(o_{t+1})\) as a stop-gradient target and trains only \(g_\omega\). In our setup the encoder receives gradients from both losses, because the forward predictor is part of the world model we deliver. Second, ICM uses a separate A3C policy with its own encoder; the ICM encoder is used only to compute a scalar intrinsic reward. Our encoder \emph{is} the world-model encoder used downstream. Conceptually, ICM motivates inverse dynamics as a way to ignore action-irrelevant variation, whereas we frame it as the anti-collapse mechanism that allows joint training of an end-to-end JEPA world model without reconstruction, EMA, or distributional regularizers.

\subsubsection{Multi-step inverse identifiability theory}
A line of theoretical work establishes asymptotic recovery of the control-endogenous latent state for multi-step inverse models in finite Exogenous Block MDPs \citep{efroni2022provable,lamb2023guaranteed,islam2022agent,mhammedi2023musik}. We summarize their assumptions precisely here.

\paragraph{AC-State \citep{lamb2023guaranteed}.} AC-State is proved correct under: (i) finite endogenous state space \(S\); (ii) finite action space \(A\); (iii) deterministic endogenous transition \(T(s' \mid s, a)\); (iv) action-independent exogenous noise; (v) bounded endogenous diameter \(D\) and a behavior policy that mixes well; (vi) a multi-step inverse loss with horizon \(K \geq D\), jointly with an information bottleneck that selects, among loss-minimizers, the encoder of minimum output cardinality. The conclusion is asymptotic identification of the partition induced by \(f^*\). The authors explicitly note that the finite-state assumption ``rules out continuous control problems''. Their Appendix~B exhibits a discrete Ex-BMDP in which a \emph{single-step} inverse model cannot distinguish two distinct endogenous states; \citet{efroni2022provable} and \citet{islam2022agent} provide further such counterexamples.

\paragraph{Refinements and known failure modes.} \citet{levine2024multistep} show that even AC-State can fail in deterministic Ex-BMDPs when the witness distance between two states exceeds \(D\), and fails for any finite horizon \(K\) under periodic endogenous dynamics; they propose adding a latent forward-dynamics term to restore guarantees.

\paragraph{Implication for our work.} Our setting differs from these results on every dimension that matters: continuous (rather than finite) state and action spaces, single-step (rather than multi-step) inverse, possibly stochastic endogenous dynamics, and no information bottleneck on the encoder's output cardinality. We therefore do not claim that any of these guarantees transfer to our inverse-dynamics objective. We cite this body of work as conceptual motivation: a substantial theoretical literature confirms that inverse-dynamics objectives of various forms learn features tied to controllable structure under \emph{some} conditions, and this informs our design choice. The single-step positive result of \citet{brandfonbrener2023inverse} for a linear-Gaussian multitask imitation model is the closest theoretical endorsement of single-step inverse dynamics that we know of, but its scope (deconfounding latent task variables in imitation) does not match ours.

\subsection{Inverse-dynamics-regularized world models}

\paragraph{DynaMo \citep{cui2024dynamo}.} DynaMo is architecturally the closest neighbor to ours: joint latent inverse and forward dynamics without reconstruction. The decisive difference is action supervision: DynaMo treats actions as unobserved latents inferred jointly with the encoder, while we supervise the inverse head with the executed action \(a_t\). Ground-truth labels turn the inverse term into a hard anti-collapse signal that a co-learned action representation cannot provide.

\paragraph{LAPO and APV.} Inverse-dynamics ideas also underlie work
that \emph{infers} pseudo-actions from action-free video. LAPO
\citep{schmidt2024learning} uses inverse dynamics model
to discover latent actions and trains a forward model conditioned on
these. APV \citep{seo2022apv} pretrains an action-free latent video
predictor and fine-tunes for control. These methods address the
complementary regime where ground-truth actions are unavailable; with
actions in hand, we use them directly to supervise inverse-dynamics
prediction.

\paragraph{World-Action Model \citep{han2026wam}.} WAM augments DreamerV2 with an inverse dynamics head and reports that this regularizes the latent state toward action-relevant structure. The construction is conceptually closest to ours among reconstruction-based methods. Two differences are decisive: WAM operates within Dreamer's pixel-reconstruction ELBO and requires reward supervision; we are reconstruction-free, JEPA-style, and reward-free. WAM treats the inverse dynamics head as one auxiliary among many; we identify it as the sufficient anti-collapse mechanism for end-to-end JEPA training.

\paragraph{EB-JEPA \citep{terver2026ebjepa}.} EB-JEPA is an action-conditioned JEPA with a multi-term loss combining variance, covariance, temporal-similarity, and inverse-dynamics regularizers; their ablations report that all four terms contribute to anti-collapse. Our claim is more specific: the inverse-dynamics term \emph{alone} suffices, removing the need for variance and covariance terms and yielding a substantially simpler training objective.

\paragraph{Auxiliary-task JEPA theory \citep{yu2025auxjepa}.} Concurrent theoretical work proves a no-unhealthy-collapse theorem for JEPAs with an auxiliary regression head: in deterministic MDPs, if both the latent prediction loss and the auxiliary regression loss are zero, then observations with different transition dynamics or different auxiliary values must map to distinct embeddings. Our IDR objective is the action-prediction instantiation of this template. Our contribution complementary to theirs is empirical: we instantiate the template on continuous pixel-action data and show that single-step IDR is sufficient in practice.

\subsection{Alternative anti-collapse mechanisms in latent world models}

\subsubsection{Reconstruction-based}

The Dreamer family \citep{hafner2019dream,hafner2020mastering,hafner2023mastering} and PlaNet \citep{hafner2019learning} train recurrent state-space models with pixel-reconstruction and reward-prediction heads under a variational objective:
\begin{equation}
    \mathcal{L}_{\text{Dreamer}} = \mathbb{E}\!\left[\sum_t -\log p_\theta(o_t | h_t, z_t) - \log p_\theta(r_t | h_t, z_t) + \beta\,\mathrm{KL}\!\bigl(q_\theta(z_t | \cdot) \,\Vert\, p_\theta(z_t | h_t)\bigr)\right].
\end{equation}
DreamerV3 retains this objective family but adds engineering refinements (symlog targets, KL balancing, two-hot value targets) for stability across domains. Reconstruction provides an implicit anti-collapse signal but commits modelling capacity to pixel-level details that are often irrelevant for control. Our setting differs from Dreamer's in two ways: we train without reward labels, and we predict in embedding space rather than reconstructing observations.

\subsubsection{EMA latent targets}

TD-MPC \citep{hansen2022temporal,hansen2024tdmpc2} learns a latent encoder \(z_t = \phi_\theta(o_t)\) and a latent transition \(g_\omega\) together with reward and value heads, all trained end-to-end via TD learning. The objective combines reward, value, and a latent-consistency term that regresses \(g_\omega(z_t, a_t)\) onto \(\texttt{sg}(\phi_{\theta^-}(o_{t+1}))\), where \(\theta^-\) is an EMA copy of \(\theta\). SPR \citep{schwarzer2021dataefficient} applies the same idea to data-efficient Atari, predicting \(k\)-step latent rollouts against a BYOL-style \citep{grill2020bootstrap} momentum target encoder. Both methods rely on the EMA target as the primary defense against collapse: a slowly-moving target makes constant solutions a moving target. TD-MPC additionally requires reward supervision. Our work avoids EMA targets and reward labels, replacing both with the inverse-dynamics signal.

\subsubsection{Augmentation and contrastive methods}

CURL \citep{laskin2020curl} adds an InfoNCE \citep{oord2018representation} contrastive auxiliary loss between augmented views of the same observation, sharing the encoder with the policy and Q-function. DrQ \citep{kostrikov2020image} shows that random image-shift augmentation alone, without an explicit auxiliary loss, suffices for state-of-the-art visual RL. Proto-RL \citep{yarats2021reinforcement} uses prototypical clustering of embeddings as both an exploration bonus and a representation learner. None of these methods learns a world model in our sense: they learn encoders shared with model-free policies. Their anti-collapse mechanisms (negatives, augmentation, prototype balance) are orthogonal to inverse-dynamics regularization and could in principle be combined with it.

\subsubsection{JEPA-style regularizers}

JEPAs predict in embedding space, with the target embedding produced by an encoder applied to the future or masked observation, typically tied to the online encoder via stop-gradient and EMA \citep{lecun2022path,assran2023self,bardes2024vjepa}. In the action-conditioned regime relevant to control, several recent methods stand out. \citet{zhou2025dinowm} (DINO-WM) freeze a pretrained DINOv2 encoder and learn only a latent dynamics model on top. \citet{sobal2025pldm} (PLDM) train end-to-end with a multi-term objective combining VICReg-style regularizers \citep{bardes2022vicreg} with an inverse-dynamics term. \citet{maes2026leworldmodel} (LeWorldModel) build on LeJEPA \citep{balestriero2025lejepa} and regularize embeddings toward isotropic Gaussianity using SIGReg. \citet{assran2025vjepa2} (V-JEPA~2) scale to large internet-video corpora and rely on EMA targets and a separate action-conditioning post-training stage. Our inverse-dynamics objective adds a distinct point on this design axis: a task-grounded action-prediction signal anchoring the embedding to controllable structure.


\end{document}